\documentclass[11pt]{article}
\usepackage{style}
\usepackage[left=1in,bottom=1in,top=1in,right=1in]{geometry}
\hypersetup{%
	colorlinks   = true, %
	urlcolor     = blue, %
	linkcolor    = blue, %
	citecolor    = blue,  %
}

\graphicspath{{img/}}

\title{A Characterization of List Language Identification in the Limit}

\author{Moses Charikar\thanks{Stanford University. Email: \texttt{moses@cs.stanford.edu.}}
\and
Chirag Pabbaraju\thanks{Stanford University. Email: \texttt{cpabbara@cs.stanford.edu.}}
\and
Ambuj Tewari\thanks{University of Michigan, Ann Arbor. Email: \texttt{tewaria@umich.edu.}}
}

\date{\today}

\usepackage[dvipsnames]{xcolor}

\begin{document}

\maketitle

\thispagestyle{empty}

\begin{abstract}

We study the problem of language identification in the limit, where given a sequence of examples from a target language, the goal of the learner is to output a sequence of guesses for the target language such that all the guesses beyond some finite time are correct.
Classical results of Gold showed that language identification in the limit is impossible for essentially any interesting collection of languages.
Later, Angluin gave a precise characterization of language collections for which language identification is possible.
Motivated by recent positive results for the related problem of language generation, we revisit the classic language identification problem in the setting where the learner is given the additional power of producing a list of $k$ guesses at each time step.
The goal is to ensure that beyond some finite time, one of the guesses is correct at each time step.
Such list learning versions of several basic learning problems have been widely studied.

We give an exact characterization of collections of languages that can be $k$-list identified in the limit, based on a recursive version of Angluin's characterization (for language identification with a list of size $1$).
This further leads to a conceptually appealing characterization: A language collection can be $k$-list identified in the limit if and only if the collection can be decomposed into $k$ collections of languages, each of which can be identified in the limit (with a list of size $1$).
We also use our characterization to establish rates for list identification in the statistical setting where the input is drawn as an i.i.d.~stream from a distribution supported on some language in the collection.
Our results show that if a collection is $k$-list identifiable in the limit, then the collection can be $k$-list identified at an exponential rate, and this is best possible.
On the other hand, if a collection is not $k$-list identifiable in the limit, then it cannot be $k$-list identified at any rate that goes to zero.

\end{abstract}

\newpage

\section{Introduction}
\label{sec:intro}

Modeling language learning is a complex and multi-faceted task. This paper is concerned with the model of \textit{language identification in the limit} proposed by Gold in his seminal work \citep{gold1967language}. One of Gold's motivations for considering such a model was to formalize the process of language acquisition in children. Viewed in a certain way, 
the language learning environment for children comprises of people around them constantly uttering valid sentences from a language, with few instances of negative examples. In this environment, a child supposedly infers a valid grammatical representation of the language, as the number of positive utterances increases. Moreover, linguistic studies indicate that children are rarely given feedback when they utter invalid sentences \citep{brown_derivational_2004}, or are not fully receptive of it when they are \citep{wiki_negative_acquisition,mcneill1970acquisition}. Thus, the learning environment would appear to comprise of a corpus of largely positive examples from the target language of interest, from which the child is supposed to infer a valid grammar for the language.

The formal setup that Gold proposes to model this is as follows: there is a countable collection $\mcC=(L_1,L_2,\dots)$ of languages, of which a certain language $L_z$ is the target language of interest. The sentences, or ``strings'' in this language are enumerated to a learner in an online fashion, and in an arbitrary order. Concretely, the learner is provided a sequence $x_1,x_2,\dots$ as input, and at every time step $t$, the learner guesses an index $i_t$ based on the input it has seen so far. The learner is said to have identified the language $L_z$ in the limit if beyond some finite time, all of its guesses $i_t$ satisfy that $L_{i_t}=L_z$. If we think of the collection $\mcC$ as comprising of different \textit{grammars} for languages, then the criterion of identification in the limit is ensuring that the learner eventually acquires a valid grammar for the target language. We note that the input does not contain any negative examples that are outside the target language, and furthermore, the learner is never given any feedback about its guess being correct or not. 

The results established by Gold for when language identification in the limit is possible in this model are drastically negative. Gold showed that essentially any interesting collection of languages that contains an infinite language cannot be identified in the limit. In particular, this rules out most formal language families of interest like regular languages, context-free languages, etc. Later follow-up work by Angluin \citep{angluin1979finding,angluin1980inductive} established a precise characterization of collections that may be identified in the limit. Namely, Angluin identified a strict combinatorial condition that must be satisfied by a collection for it to be identifiable in the limit, giving justification for why most formal language collections are not identifiable. %

In contrast to these classical negative results, the recent work of Kleinberg and Mullainathan \citep{kleinberg2024language} establishes surprisingly general positive results for a slightly different task of language \textit{generation} in the limit. In this model, instead of making a guess $i_t$ of the identity of the target language at every time step, the learner is tasked with generating a string $z_t$. The learner successfully generates in the limit, if beyond some finite time, the string $z_t$ generated by the learner belongs to the target language $L_z$, and is not one of the input strings $x_1,\dots,x_t$ it has seen so far. \cite{kleinberg2024language} showed that this modified success criterion can be achieved for every countable collection of languages! Such a strong positive result for generation in the limit motivates revisiting Gold's model for language identification in the limit, and asking: Could the strong negative results in the model be circumvented if we consider slightly more \textit{relaxed notions} of identification in the limit? 

Motivated by this question, in this work, we consider the setting of language identification in the limit, where instead of making a single guess at every time step, the learner may make a \textit{short list} of at most $k$ guesses. We will deem the learner to have successfully $k$-list identified the target language $L_z$ in the limit, if beyond some finite time, every list of guesses output by the learner \textit{contains} an index $i$ such that $L_i=L_z$. This is a natural way to impose a slightly weaker desideratum on the learner: of its $k$ guesses, we merely require that at least one is correct. In fact, this formulation of ``learning with a list'' has received considerable attention in the literature for various natural learning problems (see \Cref{sec:related-work}).

\subsection{Main Results}
\label{sec:results}

As noted above, Angluin \citep{angluin1980inductive} provided a precise characterization of language collections that are identifiable in the limit. This characterization requires every language in the collection to have a finite \textit{tell-tale} set which distinguishes it from every other language that is a proper subset of it. As our first main contribution, we present a tell-tale condition (\Cref{sec:condition}) which we term the \textit{$k$-Angluin condition}, that precisely characterizes $k$-list identification in the limit. The condition stipulates the existence of tell-tales for every language that satisfy a $k$-level recursive predicate. We show that this condition is both sufficient and necessary for $k$-list identification. Our characterization thus significantly generalizes Angluin's tell-tale-style characterization for identification in the limit, which corresponds to the special case of $k=1$, to the setting of $k$-list identification for any $k \ge 1$.

\begin{theorem}[Characterization of $k$-list Identification]
    \label{thm:k-list-identification-characterization}
    A countable collection $\mcC=(L_1,L_2,\dots)$ of languages can be identified in the limit with a list of size $k$ if and only if it satisfies the $k$-Angluin condition.
\end{theorem}

We note that our characterization above also shows that for any $k$, there exist language collections that are $k$-list identifiable in the limit, but are not $(k-1)$-list identifiable. Furthermore, there also exist countable language collections that are not $k$-list identifiable for \textit{any} finite value of $k$ (see \Cref{remark:canonical-collections,remark:separation-between-list-sizes})!

Our next result is about a structural property of $k$-list identifiable collections. As guaranteed by \Cref{thm:k-list-identification-characterization}, any $k$-list identifiable collection must satisfy the $k$-Angluin condition. As it turns out, the recursive form of this condition allows us to decompose any $k$-list identifiable collection into a union of $k$ collections that are each identifiable in the limit (in the standard $k=1$ sense)!

\begin{restatable}[Stratification of $k$-list Identifiable Collections]{theorem}{thmstratification}
    \label{thm:stratification}
    Let $\mcC=(L_1,L_2,\dots)$ be a countable language collection. Then, $\mcC$ is $k$-list identifiable in the limit if and only if $\mcC=\cup_{i=1}^k \mcC_i$ where $\mcC_i$ is identifiable in the limit for every $i\in\{1,2,\dots,k\}$.
\end{restatable}

It is instructive to contrast the two results above with standard binary classification in machine learning. Namely, in the statistical learning setting of binary classification, the ability to output a list of $k$ binary hypotheses from a hypothesis class, with the objective that one of them has low prediction error on a new test point, does not allow a learner to learn hypothesis classes that it cannot already learn with a \textit{single} hypothesis (and these are precisely those classes that have finite \textit{VC dimension} \citep{vapnik1971uniform,ehrenfeucht1989general}). Namely, if a learner outputs $k$ hypotheses of which one has low error, then the learner may as well output the single hypothesis out of these $k$ hypotheses that minimizes error on some additional training data. In contrast, in the context of identification in the limit where the language collection plays the role of a hypothesis class, the characterization in \Cref{thm:k-list-identification-characterization} shows that the ability to output a list of $k$ languages from the collection is \textit{provably more powerful} than the ability to output only a single language. 

Similarly, it is known that the union of $k$ binary hypothesis classes, each of which is statistically learnable, results in a binary hypothesis class that is also statistically learnable\footnote{Online learnability, as governed by the \textit{Littlestone dimension} \citep{littlestone1988learning,ben2009agnostic}, also satisfies such a union closure property.} (e.g., see \citep{ghazi2021near}). However, consider any collection $\mcC$ that is $k$-list identifiable, but not identifiable in the limit. \Cref{thm:stratification} ensures that such a collection may always be broken down into $k$ collections all of which are individually identifiable, but their union is \textit{not} identifiable in the limit!

We now turn towards somewhat of a more statistical setting for language identification in the limit, where the sequence of strings input to the learner is an \textit{infinite i.i.d.\ sequence} of strings drawn from a distribution supported on the target language, instead of being an arbitrary worst-case sequence. In this setting, one cares about the probability that a learner's sequence of guesses converges to the the identity of the target language. This setting was first considered by Angluin in \citep{angluin1988identifying}. More recently, it was studied by \cite{kalavasis2025limits} in the context of determining precise \textit{finite-sample rates} of identification. Namely, with any learner, one can associate a \textit{rate function}, which maps any $t$ to the probability of the learner incorrectly guessing the identity of the target language, upon seeing $t$ i.i.d.\ examples from a distribution supported on the target language.
The results of \citep{kalavasis2025limits} show that any collection that can be identified in the limit, can also be identified at an \textit{exponential rate} in the statistical setting, and this is effectively the best rate possible. Furthermore, a collection that is not identifiable in the limit does not admit \textit{any vanishing rate}.

We can also consider such a statistical setting for list identification, where the rate function measures the probability that the list output upon seeing $t$ i.i.d.\ examples does not contain the identity of the target language. In this setting, we show:

\begin{theorem}[Statistical Rates for $k$-list Identification]
    \label{thm:rates}
    Let $\mcC=(L_1,L_2,\dots)$ be any countable language collection. If $\mcC$ is $k$-list identifiable in the limit, then $\mcC$ can be $k$-list identified at an exponential rate, and this is the best rate possible. Furthermore, if $\mcC$ is not $k$-list identifiable in the limit, then $\mcC$ cannot be $k$-list identified at any rate that goes to zero.
\end{theorem}

Our result above provides a complete characterization of finite-sample statistical rates possible for $k$-list identification, analogous to the results of \cite{kalavasis2025limits} for (vanilla) identification.

\subsection{Other Related Work}
\label{sec:related-work}

\paragraph{Language Identification in the Limit.} The study of \textit{inductive inference}, starting from the works of \cite{gold1967language} and \cite{blum1975toward}, has a rich history. Various identification criteria under different names (like \textbf{EX}, \textbf{BC}, \textbf{Fex}, etc.) have been established over the years (see \cite{case1983comparison,osherson1986systems,case1999power,baliga1999synthesis}), of which \textbf{EX} (Explanatory) and \textbf{BC} (Behaviorally Correct) are arguably the two most basic notions. %
At a high level, the \textbf{EX} criterion requires the stronger guarantee that the sequence of guesses converges to the \textit{same} guess (which is also the correct answer)
beyond a finite time on every input sequence. On the other hand, the \textbf{BC} criterion does not require the sequence of guesses itself to converge to a single guess, but only requires that every guess beyond a finite time be a correct guess (since there could be multiple objects in the collection at different indices that all represent the target object). In this regard, the identification criterion that we are concerned with in this work (and the one described in the introduction above) is the \textbf{BC} criterion.

The most general setup for inductive inference involves a \textit{complete presentation} of the graph of a function $f:\N\to\N$ as input, which is any infinite sequence of pairs $(x_i,f(x_i))_{i \in \N}$, such that for every $j \in \N$, $x_i=j$ for some $i$. Notably, the function's values on the entire domain are eventually revealed in the input. This is \textit{unlike} the case in language identification in the limit, where the input only comprises of \textit{positive presentations}; namely, if we associate every language $L$ in the collection with its indicator function $f_L:x \mapsto \Ind[x \in L]$ over the universe, then the input in the model of language identification can be seen as only comprising of those $(x,f(x))$ pairs for which $f(x)=1$. We will distinguish these two paradigms in the study of inductive inference explicitly by referring to them as ``function identification" and ``language identification''.

The works most directly relevant to our setting are those on \textit{identification with a team}, introduced by \cite{smith1982power} for function identification, and extended to language identification in the works of \cite{jain1990language,jain1995aggregating,jain1995team,jain1996computational,jain1996team,jain2000team}. In particular, the notion of $\textbf{Team}^{\mathbf{m}}_{\mathbf{n}}\textbf{ID}$ identification considered in these works is defined as follows: of $n$ machines working to identify the target function/language, at least $m$ of them \textbf{ID}-identify the target language, where \textbf{ID} could be one of various identification criteria of interest (like \textbf{EX}, \textbf{BC}, \textbf{TxtEX}, \textbf{TxtBC}, etc.). Translating to these terms, our criterion for $k$-list identification would be equivalent to the notion of $\textbf{Team}^{\mathbf{1}}_{\mathbf{k}}\textbf{TxtBC}$ identification considered in the latter set of works, provided that the machines are also allowed to communicate with each other. 
More importantly though, we note that the primary focus in these works was to: 1) study and establish \textit{separations} between collections that can be $\textbf{Team}^{\mathbf{m}}_{\mathbf{n}}\textbf{ID}$ identified, but not $\textbf{Team}^{\mathbf{m'}}_{\mathbf{n'}}\textbf{ID}$ identified for $(m,n) \neq (m',n')$, and 2) study connections of team identification with the notion of \textit{probabilistic inductive inference} \citep{pitt_thesis,pitt1989probabilistic}. In contrast, our primary focus in this work is on establishing a \textit{clean} and \textit{exact} Angluin-style tell-tale characterization of $k$-list identification in the limit, as well as to establish precise finite-sample rates for the same. To the best of our knowledge, prior to our work, such an exact characterization for $k$-list identification has not been previously established.

\paragraph{Language Generation in the Limit.}

The inductive inference paradigm has recently regained a lot of attention, largely due to the milestone work of \cite{kleinberg2024language} on language generation in the limit. Since then, there have been a number of follow-ups on this topic in a very short span of time \citep{kalavasis2025limits,li2024generation,charikar2024exploring,kalavasis2024characterizations,papazov2025learning,LangGenCOLT2025}. As mentioned earlier, our motivation to revisit relaxed notions of identification is inspired from the overwhelmingly positive results possible in generation.

\paragraph{List Learning.}
List prediction, or ``list-decodable learning'' \citep{balcan2008discriminative}, where the objective is to output a short list of predictions with the goal that at least one of the outputs in the list is accurate, is also a widely studied topic, particularly in settings where data corruptions abound, or when the statistical/computational nature of the problem is pathological in the worst case. Among many works, this includes the extensive literature on list decoding in coding theory \citep{guruswami2007listdecoding}, works on robust mean estimation \citep{charikar2017learning,diakonikolas2018list}, list classification \citep{charikar2023characterization,moran2023list} as well as list regression \citep{karmalkar2019list,raghavendra2020list,pabbaraju2025characterization}. %

\section{Overview of Techniques}
\label{sec:overview}

We will now give a detailed overview of the various techniques used in establishing our results.

\paragraph{Characterization of $k$-list Identification.}

For the purposes of this overview, we will restrict ourselves to the case of $2$-list identification, and work with canonical language collections that capture the crux of our characterizing condition, as well as our upper and lower bounds. Consider first the following language collection:

\begin{example}
    \label{example:Z-i}
    Let $\mcC$ be the language collection whose constituent languages are $\Z$, as well as $\Z \setminus \{i\}$ for every $i \in \Z$.
\end{example}

The collection in \Cref{example:Z-i} is not identifiable in the limit (with a single guess), and it is helpful to see the direct diagonalization argument for this. Here, we will use the term ``identifier'' to denote a learner that seeks to identify in the limit. Fix any enumeration of the language $\Z$ (e.g., $0,-1,1,-2,2,\dots$), and fix any $n_1 \in \Z$. %
Consider an adversary that starts enumerating $\Z \setminus \{n_1\}$ in the enumeration order that was fixed for $\Z$. For any valid identifier, %
upon seeing this sequence as input, there must exist a finite time $t_1$ at which it outputs a guess $i_{t_1}$, such that $L_{i_{t_1}}=\Z \setminus \{n_1\}$. At $t_1+1$, the adversary inputs the leftmost number in the enumeration of $\Z$ that has not been enumerated as yet. 
At this point, note that the adversary has only enumerated finitely many numbers. In particular, there is a number $n_2$ that the adversary has not yet shown in the input, such that the language $\Z \setminus \{n_2\}$ is consistent with the input presented so far. Hence, the adversary can pretend that it was enumerating $\Z \setminus \{n_2\}$ in the order of $\Z$ all this time, and switch to enumerating $\Z \setminus \{n_2\}$ beyond time $t_1+1$. There must now be a finite time $t_2$ at which the identifier outputs a guess $i_{t_2}$ that satisfies $L_{i_{t_2}}=\Z \setminus \{n_2\}$, at which point the adversary can again output the leftmost number in the enumeration of $\Z$ that has not been enumerated as yet, and then switch to a different $\Z \setminus \{n_3\}$ beyond that. By repeating this ad infinitum, the adversary completely enumerates $\Z$ (since before each ``switch point'', the adversary inputs the leftmost number in $\Z$ that has not yet been enumerated), while also inducing an infinite sequence of time steps $t_1 < t_2 < \dots$ at which the identifier does not guess $\Z$ to be the target language. Thus, the identifier fails to identify $\Z$ in the limit on this adversarially constructed input sequence.

Observe however that there is a simple identifier that identifies the collection in \Cref{example:Z-i} with a list of size 2. This identifier always maintains $\Z$ as one of its guesses, and constructs its second guess as follows. The identifier fixes an enumeration $\sigma$ of $\Z$ for itself (e.g., $\sigma=0,-1,1,-2,2,\dots$), and outputs as its second guess the language $\Z \setminus \{i\}$, where $i$ is the \textit{leftmost} number in $\sigma$ that has \textit{not} yet shown up in the input. 
Observe that if the target language chosen by the adversary was $\Z$, then the identifier's first guess renders it to be correct right from the first time step. On the other hand, if the adversary chose some $\Z \setminus \{i\}$ as its target language, then every number before $i$ in $\sigma$ will eventually show up in the input, and $i$ will never show up. Thus, the identifier's second guess will stabilize to $\Z \setminus \{i\}$ after all the numbers before $i$ in $\sigma$ have appeared in the input. In either case, we have argued that the identifier $2$-list identifies the collection in the limit.

Now consider instead the following collection:

\begin{example}
    \label{example:Z-i-j}
    Let $\mcC$ be the language collection whose constituent languages are $\Z$, $\Z \setminus \{i\}$ for every $i \in \Z$, as well as $\Z \setminus \{i,j\}$ for every $i,j \in \Z, i < j$.
\end{example}

Since this collection is a superset of the collection in \Cref{example:Z-i}, it is also not identifiable in the limit by the same diagonalization argument sketched above. Could this collection be identified with a list of size $2$? Let us try to come up with an identifier similar to the one above, where the first guess is maintained to be $\Z$, so that there is protection from diagonalization against $\Z$. However, having the second guess be $\Z \setminus \{i\}$ for the leftmost $i$ that has not yet shown up in the input would make the guess converge to $\Z \setminus \{i\}$, \textit{even} when the adversary is enumerating $Z \setminus \{i,j\}$ for some $j$ that is later than $i$ in the enumeration $\sigma$ maintained by the identifier. We may then try to maintain a \textit{joint} enumeration $\sigma'$ of the set $\Z$ together with all the pairs $(i,j)$ for $i < j$, and then output $\Z \setminus F$ for the leftmost $F$ (which is either some $\{i\}$ or $\{i,j\}$) satisfying that $F$ is not a subset of the input. Unfortunately, this strategy fails too: any $i$ must be located at some finite position $p_i$ in $\sigma'$, which means that there is some $(i,j)$ that appears later than $p_i$ in $\sigma$. The adversary may be enumerating $\Z \setminus \{i,j\}$, but the identifier's second guess will either converge to $\Z \setminus \{i\}$ or $\Z \setminus \{j\}$, depending on whether $i$ or $j$ comes first in $\sigma'$. Essentially, we want the identifier to be able to place \textit{all} the $(i,j)$ pairs for a particular $i$ before $i$, which is not possible.

As it turns out, the collection in \Cref{example:Z-i-j} is \textit{not} $2$-list identifiable. However, the diagonalization argument to show this is more intricate. In particular, observe that the choice of language that was adversarially enumerated to preclude identification in the limit in the case of \Cref{example:Z-i} was clear: it was the language $\Z$, and we capitalized on being able to confuse the identifier to guess a proper subset of $\Z$ at infinitely many time steps. Note however that a 2-list identifier may always have one of its guesses be $\Z$! Namely, the identifier may completely hedge against being diagonalized on $\Z$, in which case the adversary has to look towards some other language in the collection that it can enumerate in an adversarial manner. This is where having all the $Z \setminus \{i,j\}$ languages in addition to the $\Z \setminus \{i\}$ languages will come in handy. That is, just like how the adversary could use infinitely many $\Z \setminus \{i\}$ languages to fool against $\Z$, it can also use infinitely many $\Z \setminus \{i,j\}$ languages to fool against $\Z \setminus \{i\}$. Since the identifier cannot simultaneously ``cover all the three bases'' of $\Z, \Z \setminus \{i\}$ and $\Z \setminus \{i,j\}$ with just two guesses, it will necessarily be diagonalized against on one of the bases.

\paragraph{Lower Bound.} Concretely, consider the following adversarial strategy. The adversary fixes an enumeration of all the languages in the collection, and will keep switching between enumerating different languages; in the limit however, they will have enumerated some fixed language. Whenever the adversary switches to/resumes enumerating a language, it picks back up from the leftmost not-yet-enumerated number. So suppose that: $(\star)$ the adversary starts enumerating $\Z$. For any valid 2-list identifier, there must exist a finite time $t_{1}$ at which one of the two guesses made by the identifier is $\Z$. At this point, since the adversary has only enumerated finitely many numbers, they may: $(\star\star)$ switch to enumerating any consistent $\Z \setminus \{i\}$. Again, there must exist a finite time $t_2$ at which one of the two guesses made by the identifier is $\Z \setminus \{i\}$. Crucially, at this time, the adversary \textit{inspects} the other guess made by the identifier. If this guess is not $\Z$, then the adversary has identified a time step at which none of the two guesses made by the identifier are $\Z$, and it can go back to $(\star)$ and resume enumerating $\Z$. Otherwise, the second guess made by the identifier at this time is still maintained to be $\Z$. In this case, observe that up until $t_2$, the adversary has only enumerated finitely many strings from $\Z \setminus \{i\}$. So, they can safely switch to enumerating any $\Z \setminus \{i,j\}$ that is consistent with the input so far. There must now be a finite time $t_3$ at which the identifier has $\Z \setminus \{i,j\}$ as one of its two guesses. But now the identifier is trapped, in that it must have let go of one of $\Z$ or $\Z \setminus \{i\}$ as its other guess---this is what the adversary capitalizes on. If the second guess is still $\Z$, then the adversary goes back to $(\star\star)$, and resumes enumerating $\Z \setminus \{i\}$. Otherwise, the adversary goes back to $(\star)$, and resumes enumerating $\Z$.

If the adversary repeats this routine ad infinitum, in the limit, one of two cases happens: (i) Either the adversary stops jumping back to $(\star)$ entirely, and keeps jumping back to $(\star \star)$ infinitely often, in which case it will have produced a valid enumeration of some fixed $\Z \setminus \{i\}$. But every jump back to $(\star\star)$ witnesses a time step at which the identifier did not have $\Z \setminus \{i\}$ in its two guesses, which means that the identifier does $2$-list identify $\Z \setminus \{i\}$ in the limit. (ii) Or, the adversary jumps back infinitely often to $(\star)$. In this case, the adversary produces an infinite enumeration of $\Z$. But every jump back to $(\star)$ witnesses a time step where the identifier did not have $\Z$ in its two guesses. Thus, the identifier does not 2-list identify $\Z$ in the limit.

In any case, the adversarial strategy is guaranteed to make the 2-list identifier fail on some language --- either $\Z$ or $\Z \setminus \{i\}$. Crucially, observe how the language that the adversary uses to diagonalize is \textit{dependent} on the 2-list identifier they are interacting with, and only gets determined in the limit. This is fundamentally different from the adversarial strategy used in \Cref{example:Z-i}, where the adversary can commit to diagonalizing with $\Z$ against all identifiers.

\paragraph{The Characterizing Condition.} We can inspect the structure in the collection used in \Cref{example:Z-i-j} more closely to extract a characterizing condition that prevents 2-list identification. Recall that the adversarial strategy finally fooled the identifier on either $\Z$ or some $\Z \setminus \{i\}$. This adversarial strategy was made possible because of the following property: at any finite time in the process of enumerating $\Z$, the adversary could switch to a \textit{proper subset} of $\Z$ (namely some $\Z \setminus \{i\}$) that was consistent with the input so far (i.e., which contained the finite input enumerated so far). By itself, this property is sufficient for diagonalizing against standard identification in the limit. However, the additional crucial property that enabled diagonalizing against identification with two guesses, was the fact that at any finite time in which the adversary was also enumerating some $\Z \setminus \{i\}$, they could safely switch to enumerating a \textit{further proper subset} of $\Z \setminus \{i\}$ (namely some $\Z \setminus \{i,j\}$). In other words, the underlying property that enabled the diagonalization was: there exists a language in the collection (namely $\Z$), such that for every finite subset $T$ of it, there exists a language $L'$ in the collection (namely $\Z \setminus \{i\}$) that contains $T$ \textit{and} is a proper subset of $L$, such that furthermore, for every finite subset $T'$ of $L'$, there exists a language $L''$ in the collection (namely $\Z \setminus \{i,j\}$) that contains $T''$ \textit{and} is a proper subset of $L'$. This is precisely the negation of the condition that characterizes $2$-list identification, and the generalization of this condition to $k$ levels is precisely the negation of the $k$-Angluin condition that we state in \Cref{sec:condition}.

\paragraph{Upper Bound.} Let us now consider a collection that does not have the pathological structure that the collection in \Cref{example:Z-i-j} has. Namely, every language $L$ contains a finite ``first-level tell-tale'' set $T$, such that every language $L'$ that is a proper subset of $L$ necessarily satisfies one of two properties: (1) Either $L'$ does not contain $T$, or (2) $L'$ itself contains a finite ``second-level tell-tale'' $T'$ such that every language $L''$ that is a proper subset of $L'$ does not contain $T'$. In this case, there is a natural way to utilize the existence of such layered tell-tales and state a recursive 2-list identification algorithm. The algorithm is a generalization of Angluin's algorithm given in \cite{angluin1980inductive}, and operates as follows: Suppose the target language is language $L_z$. At any time step $t$, the algorithm finds the leftmost language $L$ in the collection that is consistent with the input, and whose first-level tell-tale $T$ is entirely contained in the input. We can then see that beyond a large enough time, this language $L$ stabilizes to the \textit{leftmost} language in the collection that is a \textit{superset} of the target language $L_z$, and whose first-level tell-tale is contained in $L_z$. If the collection were in fact identifiable in the limit, we can already stop here: Angluin's condition would immediately imply that $L_z$ cannot be a proper subset of $L$, meaning that $L$ must equal $L_z$. But in our case, it \textit{is} possible for a language $L$ to have proper subset languages that contain its tell-tale. Nevertheless, the condition above stipulates that any such language $L$ better contain a \textit{second-level} tell-tale. So, we continue reasoning as follows: we set $L$ to be the first of the two guesses made by the algorithm, and then recursively run the same process on all the languages to the \textit{right} of $L$, that are proper subsets of $L$, and contain its first-level tell-tale $T$! It is not too hard to see that by virtue of the condition, either $L$ was equal to $L_z$ in the first place, or the second-level tell-tale of $L_z$ suffices to discern it from all the other proper subsets of $L$ that contained its first-level tell-tale.

\paragraph{Stratification of the Collection.} The recursive algorithm above, as well as the layered nature of our condition, naturally also inspire a greedy procedure to stratify a $2$-list identifiable collection into two separate identifiable collections. By the characterization above, any $2$-list identifiable collection must satisfy our condition of having suitable first-level and second-level tell-tales. The greedy procedure is then as follows: let us construct a directed graph over the languages in the collection, where we draw a directed edge $L \to L'$ for every pair $L,L'$ which satisfies that $L'$ is a proper subset of $L$, and $L'$ contains the first-level tell-tale of $L$. Intuitively, for any such pair $L, L'$, the first-level tell-tale does not suffice to discern between them. So, we set $\mcC_1 \subseteq \mcC$ to comprise of all the languages $L \in \mcC$, that do not have any directed edge coming into it. Then, for any pair of languages $L,L'$ in $\mcC_1$, it must be the case that neither of them contain the other's first-level tell-tale (otherwise, one of them would have a directed edge coming into it, and would not be added to $\mcC_1$). In other words, within $\mcC_1$, the first-level tell-tales of every language suffice as tell-tales in the standard version of Angluin's condition; hence, $\mcC_1$ is identifiable in the limit. On the other hand, observe also that any language $L'$ in $\mcC \setminus \mcC_1 := \mcC_2$ must also contain a second-level tell-tale, by virtue of it satisfying our nested condition. This is because every language $L' \in \mcC_2$ satisfies that it is a proper subset of some other language $L$, and that it also contains the first-level tell-tale of $L$. Thus, within $\mcC_2$, the second-level tell-tales of all languages suffice as tell-tales in the standard version of Angluin's condition, making $\mcC_2$ also identifiable in the limit. We have thus partitioned $\mcC$ into $\mcC_1$ and $\mcC_2$ in a way that both $\mcC_1$ and $\mcC_2$ are identifiable in the limit. We can repeat this greedy peeling procedure $k$ times over residual collections to generalize to the case where the collection is $k$-list identifiable in the limit.

\paragraph{Statistical Rates.} We will now describe our results on finite-sample rates in the statistical setting. These results extend those of \cite{kalavasis2025limits} for identification pointwise to the list identification setting.

We will first argue that any collection that is $k$-list identifiable in the limit can be $k$-list identified at an exponential rate. From our stratification result above, we know that if $\mcC$ is $k$-list identifiable, it can be decomposed into $k$ individually identifiable collections. Since \cite{kalavasis2025limits} show that every identifiable collection can be identified at an exponential rate, this is already sufficient to conclude that $\mcC$ may be $k$-list identified at an exponential rate: simply break $\mcC$ into $\mcC_1,\dots,\mcC_k$, and concatenate the outputs of individual identifiers (that achieve an exponential rate) running on each of these collections. Since the target language belongs to one of the $k$ collections, the identifier running on that collection yields the desired exponential rate for list identification.

It is also straightforward to see that an exponential rate is essentially the best rate possible.\footnote{upto ``trivial'' collections that can always be list-identified with a single example, see \Cref{sec:exponential-rate-best-possible}.} In particular, suppose that the collection contains $k+1$ distinct languages that all share a string $x$, and consider a distribution that assigns mass $1/2$ to $x$. Then, there is at least a $2^{-t}$ chance that an i.i.d. sequence of $t$ examples drawn from the distribution comprises only of $x$. In this case, since the algorithm only outputs $k$ guesses at any time step, at least one of the $k+1$ languages is not identified at time $t$ if this input sequence is realized. With a couple more steps of simple reasoning, this argument can be formalized to show that a rate faster than exponential is not possible.

Similar to \cite{kalavasis2025limits}, the major bulk of the work goes into showing that a collection that is not $k$-list identifiable in the limit cannot be $k$-list identified at \textit{any} vanishing rate. Here, we note that \cite{kalavasis2025limits} were able to directly use a previous result by Angluin \citep{angluin1988identifying} in their proof, which relates identification with respect to an i.i.d. stream to the standard online setting of identification. In fact, Angluin draws this connection by using results established by Pitt \citep{pitt_thesis} in the related setting of \textit{probabilistic} identification in the limit.\footnote{In probabilistic identification, an identifier is allowed to be randomized, and we care about the probability (over the random coins of the identifier) that the guesses of the identifier converge to the target language on any fixed sequence.} Essentially, Pitt relates identification in the limit to probabilistic identification in the limit, which Angluin in turn relates to identification with respect to an i.i.d. sequence. While it sufficed for \cite{kalavasis2025limits} to directly instantiate the end result, we are not afforded this convenience, and have to carefully re-derive the entire bridge ourselves for the setting of list identification. In particular, we show a standalone result analogous to that obtained by Pitt, relating $k$-list identification to probabilistic $k$-list identification. The main part of this result, which converts a probabilistic $k$-list identifier to a standard deterministic $k$-list identifier, requires several ideas, including a careful ``$\topk$'' aggregation scheme over the predictions made by different nodes in the \textit{computation tree} of a probabilistic $k$-list identifier. We also redo Angluin's part of the analysis in more detail, relating probabilistic $k$-list identification to $k$-list identification over an i.i.d.\ stream of examples. The end result of this bridge asserts that for any collection that is not $k$-list identifiable in the limit, for any $k$-list identifier, there exists a distribution over a target language such that with probability at least $\frac{1}{k+1}$ over an i.i.d.\ stream of examples drawn from this distribution, the $k$-list identifier fails to $k$-list identify the target language in the limit. The last part of the analysis then shows how any $k$-list identifier that achieves a vanishing rate with respect to arbitrary distributions may be boosted up to $k$-list identify from an i.i.d.\ stream from any distribution with probability 1, which contradicts the previous assertion.

\paragraph{Roadmap.} With this overview, we now move on to formally deriving all our results. We first state all the required technical preliminaries in \Cref{sec:preliminaries}. We then precisely state the $k$-Angluin condition for $k$-list identification in \Cref{sec:condition}. We derive a $k$-list identification algorithm for any collection that satisfies this condition in \Cref{sec:ub}. \Cref{sec:lb} contains the diagonalization argument which shows that the condition is in fact necessary for $k$-list identification. In \Cref{sec:stratification}, we detail the structural stratification result which decomposes any $k$-list identifiable collection into $k$ identifiable collections. Finally, we derive all results related to statistical rates, together with intermediate results on probabilistic list identification, in \Cref{sec:list-identification-rates}.
\section{Preliminaries}
\label{sec:preliminaries}

$\N=\{1,2,3,\dots\}$ denotes the set of natural numbers, and $\Z=\{\dots,-2,1,0,1,2,\dots\}$ denotes the set of integers. We assume a countable collection $\mcC=(L_1,L_2,\dots)$ of languages, where every language is a subset of a countable universe $U$. An enumeration of a language $L$ is a sequence $x_1,x_2,\dots$ which satisfies that every $x_i \in L$, and furthermore, for every $x \in L$, there exists $i$ such that $x_i = x$. We only assume that every language in the collection is non-empty; otherwise, it may be finite, or countably infinite. Throughout the paper, we will reserve the variable $k$ to denote the size of the list output by a learner that is trying to do list identification, or a \textit{list identifier}.

\begin{definition}[List Identifier]
    \label{def:list-identifier}
    A list identifier with list size $k$, or a $k$-list identifier, is a function\footnotemark ~$\mcA:U^* \to \N^{k}$, which at any time step $t$, takes as input a finite ordered sequence of strings $x_1,\dots,x_t$, and outputs a list of indices $\mu_t=\mcA(x_1,\dots,x_t)=(i_1,\dots,i_k)$.
\end{definition}
\footnotetext{In this paper, we only focus on the existence of list-valued functions that can identify in the limit; we will not be concerned about the computational power required to compute this function.}

For a language $L \in \mcC$, and a list $\mu \in \N^*$, we use the notation $L \id \mu$ to denote that $\exists i \in \mu$ such that $L_i=L$. In other words, $\mu$ contains the ``identity'' of language $L$, upto equivalent copies of $L$ at different indices in $\mcC$. Similarly, $L \not \id \mu$ denotes that $\forall i \in \mu$, $L_i \neq L$. The notation $\mu \di L_z$ and $\mu \not\di L_z$ represents the corresponding relations with the order of the arguments reversed.

\begin{definition}[List Identification in the Limit]
    \label{def:list-identification-in-the-limit}
    A $k$-list identifier $\mcA$ identifies a countable collection $\mcC=(L_1,L_2,\dots)$ in the limit if for every language $L_z\in \mcC$, and for every enumeration $x_1,x_2,\dots$ of $L_z$ presented to $\mcA$ as input, there exists a finite time $t^\star$ such that for every $t \ge t^\star$, the list $\mu_t=\mcA(x_1,\dots,x_t)$ output by $\mcA$ satisfies that $L_z \id \mu_t$.
\end{definition}

We now define the notion of a \textit{valid distribution}, which is any distribution that is supported on some language in the collection.

\begin{definition}[Valid Distribution]
    \label{def:valid-distribution}
    A distribution $\mcD$ over $U$ is said to be valid for a language $L$ if $\mcD(x) > 0$ if and only if $x \in L$.
\end{definition}

We can then describe statistical rates that an identifier may attain with respect to valid distributions.

\begin{definition}[List Identification Rates]
    \label{def:list-identification-rates}
    For a countable language collection $\mcC$, and a rate function $R:\N \to [0,1]$ satisfying $\lim_{t \to \infty}R(t)=0$, we say that:
    \begin{itemize}
        \item $\mcC$ can be $k$-list identified at rate $R$ if there exists a $k$-list identifier $\mcA$ 
        which satisfies that for every language $L_z \in \mcC$, and for every distribution $\mcD$ that is valid for $L_z$, there exist constants
        $c_1=c_1(L_z, \mcD, \mcC)>0$ and $c_2=c_2(L_z,\mcD, \mcC) > 0$ 
        such that for every $t \in \N$,
        \begin{align}
            \Pr_{x_1,\dots,x_t \sim \mcD^t}[\mcA(x_1,\dots,x_t) \not\di L_z] \le %
            c_1 \cdot R(c_2 \cdot t).
        \end{align}
        
        \item $\mcC$ cannot be $k$-list identified at a rate faster than $R$ if for every $k$-list identifier $\mcA$, there exists a language $L_z \in \mcC$, a distribution $\mcD$ valid for $L_z$, %
        and constants 
        $c_1=c_1(L_z, \mcD, \mcC)>0, c_2=c_2(L_z,\mcD, \mcC) > 0$
        such that for infinitely many $t \in \N$,
        \begin{align}
            \Pr_{x_1,\dots,x_t \sim \mcD^t}[\mcA(x_1,\dots,x_t) \not\di L_z] \ge %
            c_1 \cdot R(c_2 \cdot t).
        \end{align}
    \end{itemize}
    Furthermore, we have the following:
    \begin{itemize}
        \item (Optimal Rate) Rate $R$ is the optimal rate at which $\mcC$ can be $k$-list identified if $\mcC$ can be $k$-list identified at rate $R$, but cannot be $k$-list identified at a rate faster than $R$.
        \item (No Rate) $\mcC$ cannot be $k$-list identified at any rate if for every $k$-list identifier $\mcA$, there exists a language $L_z \in \mcC$ and a distribution $\mcD$ valid for $L_z$ such that 
        \begin{align}
            \limsup_{t \to \infty}\Pr_{x_1,\dots,x_t \sim \mcD^t}[\mcA(x_1,\dots,x_t)\not\di L_z]>0.
        \end{align}
    \end{itemize}
\end{definition}

\section{The $k$-Angluin Condition}
\label{sec:condition}

We begin by precisely stating the $k$-Angluin condition that characterizes identification in the limit with a list, and plays a central role throughout the rest of the paper. Fix a countable language collection $\mcC=(L_1,L_2,\dots)$. For language $L_i$ at index $i$ in $\mcC$ %
and list size parameter $k \ge 1$, define the predicate $\Psi(L_i, k)$\footnote{We remark that the predicate $\Psi$ has an implicit dependence on $\mcC$ (which we have fixed); we suppress this dependence in the notation for readability.} inductively as follows:

\begin{align}
    &\textbf{Base Predicate. } 
    \nonumber \\
    &\Psi(L_i, 1) := \exists \text{ finite } T^{(1)}_{i} \subseteq L_i \text{ such that } \forall j \in \N, \text{ if } L_j \subsetneq L_i \text{ then }  T^{(1)}_{i} \nsubseteq L_j. \label{eqn:base-predicate}\\
    \nonumber \\
    &\textbf{Inductive Predicate. } 
    \forall k > 1: %
    \nonumber \\
    &\Psi(L_i, k) := \exists \text{ finite } T^{(k)}_{i} \subseteq L_i \text{ such that } \forall j \in \N, 
    \text{ if } L_j \subsetneq L_i \text{ then } \text{either }T^{(k)}_{i} \nsubseteq L_j \text{ or } \Psi(L_j, k-1) \label{eqn:inductive-predicate}.
\end{align}
The precise ``$k$-Angluin'' condition that characterizes $k$-list identification is then as follows:
\begin{align}
    \label{eqn:condition}
    \boxed{
        \forall i \in \N,\; \Psi(L_i, k).
    }
\end{align}

The sets $T^{(k)}_i$ are known as \textit{tell-tale} sets for the language $L_i$. We note that \eqref{eqn:condition} instantiated with $k=1$, i.e., simply having the base predicate \eqref{eqn:base-predicate} hold for every language is precisely Angluin's condition that characterizes standard language identification in the limit without a list \citep{angluin1980inductive}.

\begin{remark}[Uniqueness of tell-tales across levels]
    \label{remark:uniqueness-of-tell-tales}
    Suppose we have that $\Psi(L_1, 2)$ and $\Psi(L_2, 2)$ for two languages $L_1$ and $L_2$ in $\mcC$. Then, consider some $T^{(2)}_1$ and $T^{(2)}_2$ that satisfy the respective predicates, and suppose it is the case that language $L_3 \in \mcC$ is a proper subset of both $L_1$ and $L_2$, and also contains both of $T^{(2)}_1$ and $T^{(2)}_2$. Then, according to \eqref{eqn:inductive-predicate}, in order for $\Psi(L_1, 2)$ to be true, $\Psi(L_3, 1)$ needs to be true, which implies the existence of a $T^{(1)}_3 \subseteq L_3$ satisfying \eqref{eqn:base-predicate}. Similarly, in order for $\Psi(L_2, 2)$ to be true, $\Psi(L_3, 1)$ needs to be true again, which, strictly speaking, implies the existence of a $\tilde{T}^{(1)}_3 \subseteq L_3$ satisfying \eqref{eqn:base-predicate}. Notably, observe that the definition allows for $T^{(1)}_3$ and $\tilde{T}^{(1)}_3$ to be different: as long as \eqref{eqn:base-predicate} is satisfied for both, there is no contradiction to either of $\Psi(L_1, 2)$ or $\Psi(L_2, 2)$. Nevertheless, observe that we can have one of $T^{(1)}_3$ or $\tilde{T}^{(1)}_3$ simultaneously satisfy both the ``separate instantiations'' of  $\Psi(L_3, 1)$. 
    Therefore, without loss of generality, whenever we consider an instantiation of the predicate $\Psi(L_i, k)$ to hold, we can assume it to hold with a unique tell-tale $T^{(k)}_i$.
    We note also that if $\Psi(L_i, k)$ holds for any language $L_i$, the tell-tale $T^{(k)}_i$ that makes $\Psi(L_i, k)$ hold also suffices to make  $\Psi(L_i, k+1)$ hold. Thus, for any language $L_i$, the tell-tale $T^{(k)}_i$ for the smallest $k$ for which $\Psi(L_i, k)$ holds suffices as a tell-tale for every predicate $\Psi(L_i, k')$ where $k' \ge k$. 
\end{remark}

\begin{remark}[Canonical Collections]
    \label{remark:canonical-collections}
    It is helpful to keep in mind the suite of canonical collections $(\mcC_k)_{k \ge 1}$, where $\mcC_k = \{\Z\} \cup \{\Z \setminus F: F \subseteq \Z, |F| \le k\}$. Namely, every language in $\mcC_k$ excludes a finite subset $\Z$ of size at most $k$. We can observe that $\mcC_k$ does not satisfy the $k$-Angluin condition, but satisfies the $(k+1)$-Angluin condition. To see why $\mcC_k$ does not satisfy the $k$-Angluin condition, observe that no matter what finite subset $T \subseteq \Z$ we try to ascribe to the language $\Z$ for the purposes of satisfying $\Psi(\Z,k)$, there are many languages $\Z \setminus \{i\}$ that are proper subsets of $\Z$, contain $T$, and also do not satisfy $\Psi(\Z \setminus \{i\}, k-1)$. On the other hand, any arbitrary assignment of finite tell-tale sets $T_i^{(k')}$ for every language $L_i \in \mcC_k$ and $k' \le k$ suffices for the purposes of satisfying the $(k+1)$-Angluin condition for $\mcC_k$. To see this, notice that once we fix a tell-tale set $T_i^{(k')}$ for some language $L_i=\Z \setminus F$ towards the purpose of satisfying $\Psi(L_i, k')$, the languages $L_j$ that are proper subsets of $L_i$ necessarily satisfy that $L_j = \Z \setminus G$ for $|G| < |F|$. Hence, the languages $L_i$ for which we recursively hit the base predicate $\Psi(L_i, 1)$ all satisfy that $L_i = \Z \setminus F$ where $|F|=1$, and no two languages of this form are proper subsets of each other. We also note that the collection $\mcC_\infty = \{\Z\} \cup \{\Z \setminus F: F \subseteq \Z, |F| < \infty \}$ does not satisfy the $k$-Angluin collection for any finite value of $k$.
\end{remark}
\section{Upper Bound}
\label{sec:ub}

We will now show that any countable collection that satisfies the $k$-Angluin condition can be $k$-list identified in the limit. Towards this, consider \Cref{algo:list-identification} for $k$-list identification.

\begin{algorithm}[t]
    \SetAlgoLined
    \caption{List Identification Algorithm}\label{algo:list-identification}
    \KwIn{List size $k$, set of indices $I$ satisfying $\forall i \in I:\Psi(L_i, k)$, finite dataset $S$}
    \KwOut{A list of size at most $k$}
    \SetKwFunction{ListIdentify}{ListIdentify}
    \SetKwProg{Proc}{Procedure}{:}{}
    \Proc{\ListIdentify{$k, I, S$}}{
        \If{$I$ is empty or $k=0$}{
            \Return{$\emptyset$}
        }
        \nl $i^\star \gets \min\left\{i \in I: S \subseteq L_i \text{ and } S \supseteq T^{(k)}_{i}\right\}$\footnotemark \label{line:L} \\
        \If{$i^\star = \infty$}{
            \Return{$\{1\}$}
        }
        \nl $I' \gets \{j \in I: j > i^\star \text{ and } L_j \subsetneq L_{i^\star} \text{ and } T^{(k)}_{i^\star} \subseteq L_j\}$ \label{line:C}\\
        \Return $\{i^\star\} \, \cup$ \texttt{ListIdentify}$(k-1, I', S)$
    }
\end{algorithm}
\footnotetext{We define $\min$ over an empty set to be $\infty$.}

Observe that for any $k,I,S$, the size of the list output by $\texttt{ListIdentify}(k,I, S)$ is at most $k$. This is because every time we recursively invoke \texttt{ListIdentify}, we append a language to the output, but also decrease $k$ by 1; furthermore, we return $\emptyset$ when $k=0$. This establishes that the size of the list output is at most $k$.

We now claim that when the algorithm above is invoked on increasing prefixes of an input enumeration, the index $i^\star$ determined in \Cref{line:L} stabilizes beyond some finite time step.

\begin{claim}[$i^\star$ stabilizes]
    \label{claim:L-stabilizes}
    Let the index of the target language being enumerated be $z$, and let the enumeration be $(x_1,x_2,\dots)$. Let $I$ be any set of indices that contains $z$, and let $k \ge 1$ be such that $\forall i \in I,\; \Psi(L_i, k)$. Denote by $S_t$ the set of distinct strings in a finite prefix $(x_1,\dots,x_t)$ of the enumeration. Then, there exists a finite time step $t^\star$, such that for every $t \ge t^\star$, the index $i^\star$ determined in \Cref{line:L} of \Cref{algo:list-identification} upon invoking \textup{\texttt{ListIdentify}}$(k,I, S_t)$ satisfies
    \begin{align}
        \label{eqn:i-star-stabilized-value}
        i^\star = \min\left\{i \in I: i \le z \text{ and } L_i \supseteq L_z \text{ and } L_z \supseteq T^{(k)}_i\right\}.
    \end{align}
    Note that $i=z$ satisfies the condition, and hence $i \neq \infty$ for $t \ge t^\star$.
\end{claim}
\begin{proof}

    Consider first any $i \in I$ for which $i \le z$ and $L_i \supseteq L_z$, but $L_z \nsupseteq T^{(k)}_i$. Note that such an $L_i$ never satisfies $S_t \supseteq T^{(k)}_i$ for any $t$. This is simply because $T^{(k)}_i$ contains some string $x \notin L_z$ which will never show up in the input. Thus, such an $i$ never becomes feasible in \Cref{line:L}.

    Now consider the set $J=\{i \in I: i \le z \text{ and } L_i \supseteq L_z \text{ and } L_z \supseteq T^{(k)}_i\}$, which is precisely the set in \eqref{eqn:i-star-stabilized-value}. Note first that $J$ is a finite and non-empty set, since $z$ belongs to it. Furthermore, for any $i \in J$, observe that $S_t \subseteq L_i$ for every $t$, simply because the input is an enumeration of $L_z$, and $L_i \supseteq L_z$. We now claim that there exists a finite time $t_1$, such that for every $t \ge t_1$, $S_t \supseteq T^{(k)}_i$ for every $i \in J$. Again, this is true because $L_z \supseteq T^{(k)}_i$ for every $i \in J$, and since there are finitely many $i \in J$, and each $T^{(k)}_i$ is also a finite set, there is a finite time $t_1$ when all of $\cup_{i \in J} T^{(k)}_i$ shows up in the input. Thus, for every $t \ge t_1$, every $L_i$ for $i \in J$ is feasible in \Cref{line:L}.

    Finally, observe also there exists a finite time step $t_2$, such that for every $t \ge t_2$, every $i \in I$ for which $i \le z$ and $L_i \nsupseteq L_z$ satisfies that $S_{t} \nsubseteq L_i$: this is simply because for every such $i$, $L_z$ contains some string $x_i \notin L_i$, and this $x_i$ appears in the enumeration at some finite time. Since there are only finitely many such $i$, there is a finite time $t_2$ when the input will contain $x_i$ for every $i$. Thus, for every $t \ge t_2$, we have that every such $i$ is infeasible in \Cref{line:L}.

    Combining the reasoning above, we can conclude that for every $t \ge \max(t_1, t_2)$, the index $t^\star$ determined in \Cref{line:L} satisfies
    \begin{align*}
        i^\star=\min\left\{i \in I: i \le z \text{ and } L_i \supseteq L_z \text{ and } L_z \supseteq T^{(k)}_i\right\}.
    \end{align*}

\end{proof}

Note that the expression for $i^\star$ in \eqref{eqn:i-star-stabilized-value} does not depend on $t$, and hence $i^\star$ has stablized to this value for every $t \ge t^\star$.

We are now ready to argue that any collection that satisfies the $k$-Angluin condition can be $k$-list identified by \Cref{algo:list-identification}.

\begin{theorem}[Upper Bound]
    \label{thm:k-angluin-condition-ub}
    Let $\mcC=(L_1,L_2,\dots)$ be a countable collection that satisfies the $k$-Angluin condition given in \eqref{eqn:condition}. Let $L_z$ be the target language being enumerated as $(x_1,x_2,\dots)$, and let $S_t$ denote the set of distinct strings in a finite prefix $(x_1,\dots,x_t)$ of the enumeration. Then, there exists a finite $t^\star$, such that for every $t \ge t^\star$, $L_z \id$ \textup{\texttt{ListIdentify}}$(k,\N, S_t)$.
\end{theorem}
\begin{proof}
    Let $I_k=\N$. The set $I_k$ contains the index $z$ of the target language $L_z$, and since $\mcC$ satisfies the $k$-Angluin condition, we have that $\forall i \in I_k,\; \Psi(L_i, k)$. Hence, the requirements of \Cref{claim:L-stabilizes} are satisfied, implying the existence of a finite time step $t^\star_k$, such that for every $t \ge t^\star_k$, the index $i^\star_k$ determined in \Cref{line:L} of \texttt{ListIdentify}$(k,I_k, S_t)$ is
    \begin{align*}
        i^\star_k=\min\left\{i \in I_k: i \le z \text{ and } L_i \supseteq L_z \text{ and } L_z \supseteq T^{(k)}_i\right\}.
    \end{align*}
    Furthermore, $i^\star_k$ is always included in the output list for $t \ge t^\star_k$. Now, either $L_{i^\star_k}=L_z$, in which case we are done. Otherwise, $i^\star_k < z$, and $L_{i^\star_k}$ is a proper superset of $L_z$ for which $L_z \supseteq T^{(k)}_{i^\star_k}$. Now, since $i^\star_k$ has stabilized for $t \ge t^\star_k$, the set of indices determined in \Cref{line:C} also stabilizes for $t \ge t^\star_k$. Denote this set by $I_{k-1}$, and recall that 
    \begin{align*}
        I_{k-1} = \{j \in I_k: j > i^\star_k \text{ and } L_j \subsetneq L_{i^\star_k} \text{ and } T^{(k)}_{i^\star_k} \subseteq L_j\}.
    \end{align*}
    By our previous reasoning, $z \in I_{k-1}$. Furthermore, recall that $\Psi(L_{i^\star_k}, k)$ holds, and every index $j$ in $I_{k-1}$ corresponds to a language $L_j$ that is a proper subset of $L_{i^\star_k}$ and also contains $T^{(k)}_{i^\star_k}$. Recalling \eqref{eqn:inductive-predicate}, this implies that $\Psi(L_j, k-1)$ holds for every $j \in I_{k-1}$.

    We then repeat the above argument on the recursive invocation to \texttt{ListIdentify}$(k-1, I_{k-1}, S_t)$ for $t \ge t^\star_k$. As argued above, the set $I_{k-1}$ contains the index $z$ of the target language, and $\Psi(L_j, k-1)$ holds for every $j \in I_{k-1}$. The requirements of \Cref{claim:L-stabilizes} are again satisfied, implying the existence of a finite time step $t^\star_{k-1}$, such that for every $t \ge \max(t^\star_k, t^\star_{k-1})$, the index $i^\star_{k-1}$ determined in \Cref{line:L} of \texttt{ListIdentify}$(k-1, I_{k-1}, S_t)$ is
    \begin{align*}
        i^\star_{k-1}=\min\left\{i \in I_{k-1}: i \le z \text{ and } L_i \supseteq L_z \text{ and } L_z \supseteq T^{(k-1)}_i\right\}.
    \end{align*}
    Furthermore, $i^\star_{k-1}$ is always included in the output list for $t \ge \max(t^\star_k, t^\star_{k-1})$. Again, if $L_{i^\star_{k-1}}=L_z$, we are done. Otherwise, $i^\star_{k-1} < z$, and $L_{i^\star_{k-1}}$ is a proper superset of $L_z$ for which $L_z \supseteq T^{(k-1)}_{i^\star_{k-1}}$. Since $i^\star_{k-1}$ has stabilized for $t \ge \max(t^\star_k, t^\star_{k-1})$, the set of indices determined in \Cref{line:C} also stabilizes. Denote this set by $I_{k-2}$, and recall that 
    \begin{align*}
        I_{k-2} = \{j \in I_{k-1}: j > i^\star_{k-1} \text{ and } L_j \subsetneq L_{i^\star_{k-1}} \text{ and } T^{(k-1)}_{i^\star_{k-1}} \subseteq L_j\}.
    \end{align*}
    By our previous reasoning, $z \in I_{k-2}$. Furthermore, recall that $\Psi(L_{i^\star_{k-1}}, k-1)$ holds, and every index $j$ in $I_{k-2}$ corresponds to a language $L_j$ that is a proper subset of $L_{i^\star_{k-1}}$ and also contains $T^{(k-1)}_{i^\star_{k-1}}$. Recalling \eqref{eqn:inductive-predicate}, this implies that $\Psi(L_j, k-2)$ holds for every $j \in I_{k-2}$.

    Continuing thus all the way down till $k=1$, we would either have that $L_z \id (i^\star_k,\dots,i^\star_2)$ for every $t \ge \max(t^\star_k,\dots,t^\star_{2})$. Otherwise, we consider the invocation to \texttt{ListIdentify}$(1, I_1, S_t)$ for $t \ge \max(t^\star_k,\dots,t^\star_{2})$. Recall that we ensure that $\Psi(L_j, 1)$ holds for every $j \in I_1$, and that $I_1$ contains the index $z$ of the target language. Invoking \Cref{claim:L-stabilizes} one final time, there exists a finite time $t^\star_{1}$ such that for every $t \ge \max(t^\star_k, \dots, t^\star_1)$, the index $i^\star_1$ determined in \Cref{line:L} of \texttt{ListIdentify}$(1, I_1, S_t)$ is 
    \begin{align*}
        i^\star_{1}=\min\left\{i \in I_{1}: i \le z \text{ and } L_i \supseteq L_z \text{ and } L_z \supseteq T^{(1)}_i\right\}.
    \end{align*}
    Furthermore, $i^\star_{1}$ is now always included in the output list for $t \ge \max(t^\star_k, \dots, t^\star_1)$. We now claim that if we are indeed in this case, $L_{i^\star_1}$ must indeed \textit{equal} $L_z$. Otherwise, if $L_{i^\star_1}$ is a proper superset of $L_z$, it would be the case that $L_z \subsetneq L_{i^\star_i}, T^{(1)}_{i^\star_1} \subseteq L_z$, which is a contradiction to $\Psi(L_{i^\star_1}, 1)$ (see \eqref{eqn:base-predicate}). We conclude that beyond $t \ge \max(t^\star_k,\dots,t^\star_1)$, \texttt{ListIdentify}$(k, \N, S_t)$ identifies $L_z$.
\end{proof}

\section{Lower Bound}
\label{sec:lb}

We next proceed to showing that the $k$-Angluin condition is necessary for $k$-list identification in the limit.

\begin{theorem}[Lower Bound]
    \label{thm:k-angluin-condition-lb}
    Let $\mcC=(L_1,L_2,\dots)$ be a countable collection that does not satisfy the $k$-Angluin condition given in \eqref{eqn:condition}, for $k \ge 1$.
    Then, $\mcC$ cannot be identified in the limit with a list of size $k$.
\end{theorem}
\begin{proof}
    For any $k$-list identifier $\mcA$, consider the adversarial enumeration strategy specified in \Cref{algo:adversarial-enum}.
    
    \begin{algorithm}[t]
        \SetAlgoLined
        \caption{Adversarial enumeration strategy against $k$-list identifier $\mcA$}\label{algo:adversarial-enum}
        \KwIn{Ordered sequence of indices \chain, input $S$ enumerated so far}
        \KwOut{A valid enumeration $\sigma$ of some language in $\mcC$}
        \SetKwFunction{AdvEnum}{AdvEnum}
        \SetKwProg{Proc}{Procedure}{:}{}
        \Proc{\AdvEnum{$\textup{\chain}, S$}}{
            \label{line:ancestors} Suppose $\chain :=(i_1,\dots,i_\ell)$ \\
            \While{True}{\label{line:enumerate-loop}
                $x \gets$ first string in the enumeration of $L_{i_\ell}$ that hasn't yet appeared in $S$ \\
                $S \gets S \circ x$ \\
                \If{$L_{i_\ell} \id \mcA(S)$}{
                    break
                }
            }
           \If{$\forall j \in [\ell],\, L_{i_j} \id \mcA(S)$}{
                Let $i_{\ell+1} \in \N$ be such that $L_{i_{\ell+1}} \subsetneq L_{i_\ell}$, $L_{i_{\ell+1}} \supseteq S$ and $\neg\Psi(L_{i_{\ell+1}}, k-\ell)$; terminate if $k-\ell < 0$, or if no such $i_{\ell+1}$ exists \label{line:find-proper-subset}\\
                $\chain \gets (i_1,\dots,i_\ell, i_{\ell+1})$ \\
                \texttt{AdvEnum}$(\chain, S)$ \label{line:invoke-proper-subset}\\
            }
            \Else{
                $j \gets \min\{j \in [\ell]: L_{i_j} \not\id \mcA(S)\}$ \label{line:find-ancestor-to-jump}\\
                $\chain \gets (i_1,\dots,i_{j})$ \label{line:trim-ancestors}\\
                \texttt{AdvEnum}$(\chain, S)$ \label{line:invoke-ancestor}\\
            }
        }
    \end{algorithm}
    
    Now, since $\mcC$ does not satisfy the $k$-Angluin condition, there exists $i_1 \in \N$ such that $\neg \Psi(L_{i_1}, k)$ holds. Referring to \eqref{eqn:inductive-predicate}, this means that
    \begin{align*}
        \forall \text{ finite } T \subseteq L_{i_1}, \, \exists i_2 \in \N \text{ such that } L_{i_2} \subsetneq L_{i_1} \text{ and } T \subseteq L_{i_2} \text{ and } \neg \Psi(L_{i_2}, k-1).
    \end{align*}
    Here, we adopt the convention that $\Psi(L_{i_2}, 0)=$ False for every $i_2 \in \N$. 
    
    We now claim that the list-identifier $\mcA$ fails the list identification criterion for $\mcC$ on the infinite sequence $S$ populated upon invoking \texttt{AdvEnum}$((i_1), ())$. We first establish a few invariants of \Cref{algo:adversarial-enum} upon this invocation.

    \paragraph{Invariant 1:} Whenever \texttt{AdvEnum}$(\chain, S)$ is invoked, $\chain$ is a non-empty sequence of indices $(i_1,\dots,i_\ell)$, such that $L_{i_1} \supsetneq L_{i_2} \supsetneq \dots \supsetneq L_{i_\ell}$. Furthermore, we also have that $\neg \Psi(L_{i_1}, k), \neg \Psi(L_{i_2}, k-1),\dots,\neg \Psi(L_{i_\ell}, k-(\ell-1))$, and that $S \subseteq L_{i_j}$ for every $j \in [\ell]$.
    
    To see this, we argue inductively: observe that this is true in the base case when we invoke \texttt{AdvEnum}$((i_1), ())$, either vacuously or by choice of $i_1$. Now, suppose that the invariant is true when \texttt{AdvEnum}$(\chain, S)$ is invoked---we will argue that the invariant holds true at a subsequent invocation. Note that if there is a subsequent invocation at all, it must be the case that the while loop in \Cref{line:enumerate-loop} breaks. Then, if we are in the case that $\forall j \in [\ell],\,L_{i_j} \id \mcA(S)$, we must not terminate in \Cref{line:find-proper-subset}; so we will go on to append $i_{\ell+1}$ that we found to $\chain$, which is guaranteed to satisfy $S \subseteq L_{i_{\ell+1}}$, $L_{i_{\ell+1}} \subsetneq L_{\ell}$, as well as $\neg \Psi(L_{i_{\ell+1}},k-\ell)$ by construction. Thus, the invariant continues to hold true at the next invocation in \Cref{line:invoke-proper-subset}. Otherwise, if we are in the case of the else condition, we trim $\chain$ to be $(i_1,\dots,i_{j})$; then, the invariant continues to hold true at the next invocation in \Cref{line:invoke-ancestor}, due to the fact that all newly added strings to $S$ in the while loop in the present invocation were from $L_{i_\ell} \subseteq L_{i_j}$, together with the inductive hypothesis.

    The invariant above ensures that at the beginning of any invocation to \texttt{AdvEnum}$(\chain, S)$, $S$ is a valid \textit{partial} enumeration of $L_{i_\ell}$, and the while loop in \Cref{line:enumerate-loop} \textit{resumes} enumerating $L_{i_\ell}$ (namely, if the loop goes on endlessly, $S$ gets completed to be a valid \textit{complete} enumeration of $L_{i_\ell}$).

    \paragraph{Invariant 2:} Whenever \texttt{AdvEnum}$(\chain, S)$ is invoked, $|\chain| \le k+1$. 
    
    To see this, note that we begin with $|\chain|=1$ at the very first invocation \texttt{AdvEnum}$((i_1), ())$. Thereafter, assuming no termination/infinite loops, we either append a single language to $\chain$ in the case of the if condition, or we strictly \textit{reduce} the size of $\chain$ in the other case---this is because in \Cref{line:find-ancestor-to-jump}, $j$ is necessarily determined to be strictly smaller than $\ell$, since the while loop broke with $L_{i_\ell} \id \mcA(S)$. So, it suffices to argue that at an invocation of \texttt{AdvEnum}$(\chain, S)$ with $|\chain| = k+1$, the if condition, namely $\forall j \in [\ell],\, L_{i_j} \id \mcA(S)$, can never be satisfied. But this is true simply because $\mcA(S)$ is a list of size at most $k$, and the languages at the indices in $\chain$ are all distinct (e.g., by Invariant 1). This implies that at least one index $i_j \in \chain$ satisfies $L_{i_j} \not\id \mcA(S)$. Thus, we have established the invariant.

    \paragraph{Invariant 3:} The termination condition in \Cref{line:find-proper-subset} is never realized during any invocation.
    
    This follows from Invariant 1 and Invariant 2. The reasoning for Invariant 2 establishes that after the while loop breaks, the if condition $\forall j \in [\ell], \, L_{i_j} \id \mcA(S)$ can be satisfied only if $\ell=|\chain| \le k$; but this ensures that $k-\ell$ is non-negative. Furthermore, from Invariant 1, we also know that $\neg \Psi(L_{i_\ell}, k-(\ell-1))$ holds. By definition, this means that
    \begin{align*}
        \forall \text{ finite } T \subseteq L_{i_\ell}, \, \exists i_{\ell+1} \in \N \text{ such that } L_{i_{\ell+1}} \subsetneq L_{i_\ell} \text{ and } T \subseteq L_{i_{\ell+1}} \text{ and } \neg \Psi(L_{i_{\ell+1}}, k-\ell).
    \end{align*}
    Taking $T=S$ ensures the existence of the required $i_{\ell+1}$, and establishes the invariant.
    
    With these invariants established, we can now proceed towards finishing the proof. Consider first the case where the while loop in \Cref{line:enumerate-loop} never breaks in some invocation to \texttt{AdvEnum}$(\chain, S)$. By Invariant 1, this means that the loop is validly completing the enumeration of the language $L_{i_\ell}$ for $i_{\ell} \in \chain$, and $L_{i_\ell}$ is never identified in the list output by $\mcA$; this makes $\mcA$ an invalid list identifier.

    So, assume that the while loop breaks in every invocation of \texttt{AdvEnum}$(\chain, S)$. By Invariant 3 above, we know that there is never a termination. This means that there is an infinite sequence of invocations of \texttt{AdvEnum}$(\chain, S)$. Let $\chain_1, \chain_2, \dots$ be the corresponding $\chain$ parameters that each \texttt{AdvEnum} is invoked with in this sequence. Define
    \begin{align*}
        \ell^\star := \liminf_{n \to \infty} |\chain_n|.
    \end{align*}
    Note that $\ell^\star$ always exists, and furthermore belongs to the set $\{1,\dots,k+1\}$, since all the numbers in the sequence belong to this finite set. That is, there exists a large enough $n^\star$, such that for every $n \ge n^\star$, $|\chain_n| \ge \ell^\star$. This means that the indices $i_1,\dots,i_{\ell^\star}$ stay fixed in $\chain_n$ for every $n \ge n^\star$. It is also the case that there are are infinitely many $n \ge n^\star$ where the invocation \texttt{AdvEnum}$(\chain_n, S)$ satisfies $|\chain_n|=\ell^\star$. These invocations occur precisely when the previous invocation (which satisfies $|\chain_{n-1}| > \ell^\star$) determines $j=\ell^\star$ in \Cref{line:find-ancestor-to-jump}, meaning that $L_{i_{\ell^\star}} \not\id \mcA(S)$ at that time. Namely, each invocation \texttt{AdvEnum}$(\chain_n, S)$ where $|\chain_n|=\ell^\star$ can be associated to a \textit{witness} of a time step where the list output by the algorithm did not identify $L_{i_{\ell^\star}}$
    
    Finally, in each invocation where $|\chain_n|=\ell^\star$, the while loop resumes enumerating $L_{i_{\ell^\star}}$; in particular, it is adding at least one string in the enumeration of $L_{i_{\ell^\star}}$ that hasn't yet appeared in the input to $S$. This means that over the infinite execution of \Cref{algo:adversarial-enum}, $L_{i_{\ell^\star}}$ gets completely enumerated. Thus, we have argued the existence of a sequence that is a valid enumeration of some language $L_{i_{\ell^\star}} \in \mcC$, for which there are infinitely many time steps where the list output by $\mcA$ does not identify $L_{i_{\ell^\star}}$. Thus, $\mcA$ is not a valid list-identifier. 
\end{proof}

\begin{remark}
    \label{remark:separation-between-list-sizes}
    In light of \Cref{remark:canonical-collections}, we thus have that, for any $k \ge 1$, the collection $\mcC_k = \{\Z\} \cup \{\Z \setminus F: F \subseteq \Z, |F| \le k\}$ is not $k$-list identifiable, but is $(k+1)$-list identifiable in the limit. Furthermore, the collection $\mcC_{\infty} = \{\Z\} \cup \{\Z \setminus F: F \subseteq \Z, |F| < \infty\}$ is not $k$-list identifiable for \textit{any} finite value of $k$. We highlight that this is an example of a simple countable collection of languages, which, by the general positive result of \cite{kleinberg2024language} for countable collections, is \textit{generatable} in the limit, but is not identifiable in the limit, even with a finite list of guesses.
\end{remark}
\section{Stratification of $k$-list Identifiable Collections}
\label{sec:stratification}

In this section, we show that $k$-list identifiable collections admit additional structure: they can be decomposed into $k$ collections, each of which is \textit{identifiable} in the limit in the standard sense!

\thmstratification*
\begin{proof}
    One direction is straightforward: if $\mcC=\cup_{i=1}^k \mcC_i$ where $\mcC_i$ is identifiable in the limit for every $i=\{1,2,\dots,k\}$, then consider $\mcA_1,\dots,\mcA_k$ that identify $\mcC_1,\dots,\mcC_k$ respectively. We can verify that $\mcA$, which assigns $\mcA(x_1,\dots,x_t) := (\mcA_1(x_1,\dots,x_t),\dots,\mcA_k(x_1,\dots,x_t))$ successfully $k$-list identifies $\mcC$ in the limit.

    For the other direction, let $\mcC$ be a $k$-list identifiable collection. By \Cref{thm:k-angluin-condition-lb}, this must mean that $\mcC$ satisfies the $k$-Angluin condition. Namely, we have that $\forall i \in \N: \Psi(L_i, k)$.  We will construct $\mcC_k,\mcC_{k-1},\dots,\mcC_1$ satisfying $\cup_{i=1}^k\mcC_i$ in levels, using the inductive nature of the condition.

    Let $I_k=\N$, and consider the relation $R_k \subseteq I_k \times I_k$ constructed as follows:
    \begin{align}
        R_k := \{(i, j): i,j \in I_k, L_j \subsetneq L_i \text{ and } L_j \supseteq T^{(k)}_i\}.
    \end{align}
    Then, let $J_k=\{j \in I_k: (i,j) \notin R_k\; \forall i \in I_k\}$, and let $\mcC_k = \{L_j: j \in J_k\}$. Note that for any two languages $L_a, L_b$ in $\mcC_k$, if $L_b \subsetneq L_a$, then it must be the case that $L_b \nsupseteq T^{(k)}_a$; otherwise, $(a,b) \in R_k$ and $L_b$ would not have been included in $\mcC_k$. In other words, for every language $L_a \in \mcC_k$, $T^{(k)}_a$ serves as a tell-tale set for $L_a$ within $\mcC_k$, meaning that $\mcC_k$ satisfies Angluin's condition, and is identifiable in the limit.

    Now let $I_{k-1}=I_k \setminus J_k$. Observe that for any $j \in I_{k-1}$, since $L_j$ was not included in $\mcC_k$, it must be the case that there exists $i \in I_k$ for which $L_j \subsetneq L_i$ and $L_j \supseteq T^{(k)}_i$. But since $\Psi(L_i, k)$ holds for every $i \in I_k$, this must mean that $\Psi(L_j, k-1)$ holds. Thus, we have argued that $\forall j \in I_{k-1}$, $\Psi(L_j, k-1)$ holds. Furthermore, as detailed in \Cref{remark:uniqueness-of-tell-tales}, whenever we consider the predicate $\Psi(L_j, k-1)$ to hold, we can assume it to hold with a unique $T^{(k-1)}_j$. So, let us proceed with constructing the relation $R_{k-1} \subseteq I_{k-1} \times I_{k-1}$ along the same lines as above:
    \begin{align}
        R_{k-1} := \{(i, j): i,j \in I_{k-1}, L_j \subsetneq L_i \text{ and } L_j \supseteq T^{(k-1)}_i\}.
    \end{align}
    
    Then let $J_{k-1}=\{j \in I_{k-1}: (i,j) \notin R_{k-1}\; \forall i \in I_{k-1}\}$, and let $\mcC_{k-1}=\{L_j:j \in J_{k-1}\}$. Again, by the same reasoning as in the above paragraph, we have that for every language $L_a \in \mcC_{k-1}$, the set $T^{(k-1)}_a$ serves as a tell-tale for $L_a$ within $\mcC_{k-1}$, and thus, $\mcC_{k-1}$ is identifiable in the limit.

    We proceed similarly, and construct the collections $\mcC_{k-2},\dots,\mcC_1$. By the reasoning above, each of these collections is identifiable in the limit. We will now argue that their union is $\mcC$. For this, we need to show that $J_k \cup J_{k-1}\cup \dots \cup J_1 = \N$.
    
    Observe that $J_k \cup J_{k-1} \cup \dots \cup I_1 = \N$. We will argue that $J_1=I_1$, by arguing that the set $R_{1}$ is empty. Indeed, if $R_1$ contains any $(i,j)$ for $i, j \in I_{1}$, then by definition, it must be the case that $L_j \subsetneq L_i$ and $L_j \supseteq T^{1}_i$. But recall that we maintain that for every $i \in I_{1}$, $\Psi(L_i, 1)$ holds. Thus, recalling the definition of the base predicate \eqref{eqn:base-predicate}, the existence of $L_j \subsetneq L_i$ satisfying $L_j \supseteq T^{(1)}_i$ is a contradiction. We conclude that $J_1=I_1$, which means that $J_k \cup J_{k-1} \cup \dots \cup J_1 = \N$. Thus, we have that $\cup_{i=1}^k \mcC_i = \mcC$, which completes the proof.
\end{proof}

\begin{remark}
    \Cref{thm:stratification} also implies an alternative $k$-list identifier to the one described in \Cref{sec:ub}. Namely, if a collection is $k$-list identifiable in the limit, then it can be decomposed into $k$ collections each of which is identifiable in the limit. So, concatenating the outputs of identifiers run on each of these collections also successfully $k$-list identifies the collection in the limit.
\end{remark}
\section{List Identification Rates}
\label{sec:list-identification-rates}

We now proceed towards establishing rates for list identification in the setting where the input is drawn as an i.i.d. stream from a valid distribution supported on some language in the collection. The three subsequent sections respectively show that: (1) If a collection does not satisfy the $k$-Angluin condition, then it cannot be $k$-list identified at any rate, (2) A collection that satisfies the $k$-Angluin condition can be $k$-list identified at an exponential rate, and (3) A rate faster than an exponential rate is impossible for any (non-trivial) collection.

\subsection{Condition Not Satisfied $\implies$ No Rate}
\label{sec:no-rate}

The main result in this section is the following:

\begin{theorem}[$k$-Angluin Condition Not Satisfied $\implies$ No Rate]
    \label{thm:no-rate}
    Let $\mcC$ be a countable collection of languages that does not satisfy the $k$-Angluin condition \eqref{eqn:condition}. Then $\mcC$ cannot be $k$-list identified at any rate.
\end{theorem}

\subsubsection{List-identifiable Collections $\equiv$ Probabilistically List-identifiable Collections}
\label{sec:probabilistic-deterministic-equivalent}

Our discussion on $k$-list identification so far has been about list identifiers that are \textit{deterministic}, in that the list output at every time step by the identifier is a deterministic function of the input up until that time step. Here, we will show that the family of collections that are deterministically $k$-list identifiable in the limit is precisely the family of collections that are \textit{probabilistically} $k$-list identifiable in the limit. In probabilistic $k$-list identification, we allow the list output by the identifier at every time step to be randomized. Our discussion here is based on the work of \cite{pitt_thesis} on probabilistic inductive inference, although we are required to carefully derive some results differently for the purposes of list identification.

\begin{definition}[Probabilistic List Identifier]
    \label{def:probabilistic-list-identifier}
    A probabilistic list identifier with list size $k$, or a probabilistic $k$-list identifier, is a function $\mcA:\left(U \times \{0,1\}\right)^* \to \N^{k}$, which at any time step $t$, takes as input a finite ordered sequence of strings $x_1,\dots,x_t$ and the outcomes of $t$ independent and uniformly random bits $r_1,\dots,r_t$, and outputs a list of indices $\mu_t=\mcA(x_1,r_1,\dots,x_t,r_t)=(i_1,\dots,i_k)$.
\end{definition}

\begin{definition}[Probabilistic List Identification in the Limit]
    \label{def:probabilistic-list-identification-in-the-limit}
    A probabilistic $k$-list identifier $\mcA$ identifies a countable collection $\mcC=\{L_1,L_2,\dots\}$ in the limit with probability $p$ if for every language $L_z\in \mcC$, and for every enumeration $x_1,x_2,\dots$ of $L_z$, with probability at least $p$ over a random bit sequence $r_1,r_2,\dots$, there exists a finite time $t^\star$ such that for every $t \ge t^\star$, the list $\mu_t=\mcA(x_1,r_1,\dots,x_t,r_t)$ output by $\mcA$ satisfies that $L_z \id \mu_t$ 
\end{definition}

For any given input enumeration $\sigma$, we model the infinite sequence of guesses made by a probabilistic $k$-list identifier $\mcA$ on $\sigma$ as a path on a binary tree, that begins at the root, and descends down infinitely, where at a node at level $t$, the path proceeds towards the left child if the random bit $r_t=0$, and towards the right child otherwise.

\begin{figure}[t]
    \centering
        \includegraphics[scale=0.45]{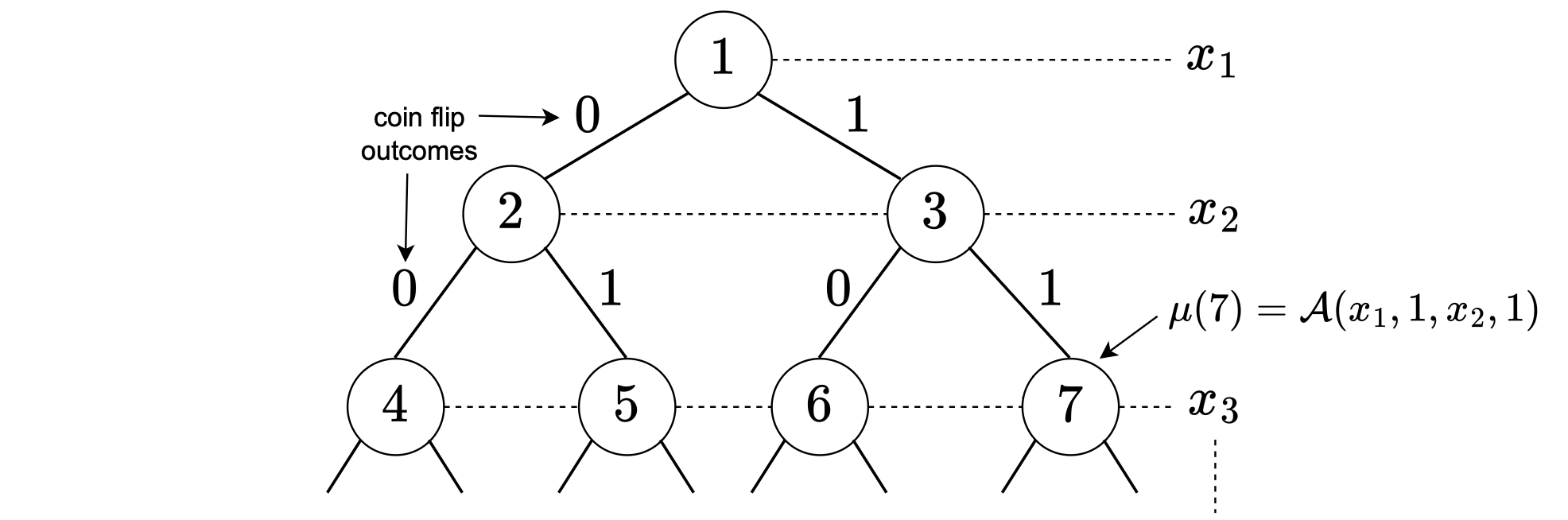}
        \caption{The tree of computation $\mcT_{\mcA, \sigma}$ of a probabilistic $k$-list identifier on an input sequence $\sigma=(x_1,x_2,\dots)$. The labels on the edges represent coin flip outcomes. The root is at level 1. Every node is associated with the list guessed by $\mcA$ at that node. 
        }
        \label{fig:computation-tree}
\end{figure}

\begin{definition}[Tree of Computation]
    \label{def:tree}
    Given a probabilistic $k$-list identifier $\mcA$ and an input enumeration $\sigma=(x_1,x_2,\dots)$, we define the object $\mcT_{\mcA, \sigma}$ to be the infinite binary tree, whose nodes are numbered $1,2,3,\dots$ starting from the root, and proceeding in a breadth-first order from left to right. A node is connected to its left child with an edge labeled 0, and to its right child with an edge labeled 1. A path in the tree is an infinite sequence of adjacent nodes $(n_1,n_2,\dots)$ starting from the root (i.e., $n_1=1$), and proceeding downwards at every step. The level of a node $n$, denoted $d(n)$, is the number of edges on the path from the root to $n$, plus one (so the root is at level 1). We associate level $i$ of the tree with input $x_i$ in the enumeration $\sigma$. We associate every node $n_i$ at level $i$ in a path, for $i > 1$, with the list $\mu(n_i) := \mcA(x_1,r_1,\dots,x_{i-1},r_{i-1})$, where $r_1,\dots,r_{i-1}$ are the edge labels on the path from the root node to $n_i$. Refer to \Cref{fig:computation-tree} for an illustration.
\end{definition}

We will now define what it means for a path in the tree of computation $\mcT_{\mcA, \sigma}$ to \textit{converge}.

\begin{definition}[Convergence in Tree]
    \label{def:convergence-in-tree}
    A path $(n_1,n_2,\dots)$ in $\mcT_{\mcA, \sigma}$ converges to a language $L_z \in \mcC$ if there exists a finite $i$ such that for every $j \ge i$, $L_z \id \mu(n_j)$. Furthermore, we say that a path $(n_1,n_2,\dots)$ converges to language $L_z \in \mcC$ at node $n_i$, for $i \ge 2$, if: (1) for every $j \ge i$, $L_z \id \mu(n_j)$, and (2) either $i=2$, or property (1) does not hold for any $j < i$.
\end{definition}
In other words, a path $(n_1,n_2,\dots)$ converges to $L_z$ at node $n_i$, if $n_i$ corresponds to the smallest $i$, such that the path converges to $L_z$ at node $n_i$. Note that because of this, any path $(n_1,n_2,\dots)$ that converges to $L_z$ converges to it at a \textit{unique} node. We denote by $P(L_z)$ the set of paths in $\mcT_{\mcA, \sigma}$ that converge to language $L_z \in \mcC$. Similarly, for any node $n \in \mcT_{\mcA, \sigma}$, we denote by $P(L_z, n)$ the set of paths in $\mcT_{\mcA, \sigma}$ that pass through $n$ and converge to language $L_z \in \mcC$ at node $n$. We can then observe that
\begin{align}
    P(L_z) = \cup_{n \in \mcT_{\mcA, \sigma}} P(L_z, n),
\end{align}
and by our previous reasoning, the sets $P(L_z, n)$ participating in the union above are all mutually disjoint.

Now consider the random experiment of tracing down a path $(n_1,n_2,\dots)$ in $\mcT_{\mcA, \sigma}$ with respect to a randomly drawn sequence of bits $R:=(r_1,r_2,\dots)$, where from any node $n_i$, we descend down the edge labeled $0$ if $r_i=0$, and down the edge labeled $1$ otherwise. Note that this defines a bijection between paths in $\mcT_{\mcA, \sigma}$ and sequences of random bits. In this case, if $\sigma$ is an enumeration of some language $L_z \in \mcC$, then the set $P(L_z)$ defines the event that $\mcA$ list-identifies $L_z$ in the limit.\footnote{The probability space is indeed well-behaved; for measure-theoretic considerations, we refer the reader to \cite{pitt_thesis}.
} 
Concretely, if $\mcA$ $k$-list identifies $\mcC$ with probability $p$, then we have that
\begin{align}
    \label{eqn:list-identification-probability-summation}
    \Pr_R[P(L_z)] = \sum_{n \in \mcT_{\mcA, \sigma}}\Pr_R \left[P(L_z,n)\right]\ge p.
\end{align}
We will require considering truncated versions of the infinite tree $\mcT_{\mcA, \sigma}$ in order to validly define a deterministic list identifier from a probabilistic list identifier. Accordingly, we require the following definition:
\begin{definition}
    \label{eqn:def-P-L_z-d}
    For any input enumeration $\sigma$, consider the infinite tree $\mcT_{\mcA, \sigma}$. For any node $n \in \mcT_{\mcA, \sigma}$ at level $d(n) \ge 2$, and for any $d \ge d(n)$, we define $P(L_z, n, d)$, for any $L_z \in \mcC$, to be the set of paths $(n_1,n_2,\dots,)$ in $\mcT_{\mcA, \sigma}$, which satisfy:
    \begin{enumerate}
        \item The path passes through node $n$, i.e., $n_i=n$ for $i=d(n)$.
        \item Either $d(n)=2$, or $\mu(n_{d(n)-1}) \not\di L_z$.
        \item Every list guessed from node $n$ onward up until level $d$ in the path identifies $L_z$, i.e., for every $j$ where $d(n) \le j \le d$, $L_z \id \mu(n_{j})$.
    \end{enumerate}
\end{definition}
In other words, $P(L_z, n, d)$ comprises of paths that appear to have converged to $L_z$ at node $n$, if one myopically looks beyond node $n$ only until level $d$. We can observe that for any $d \ge d(n)$,
\begin{align}
    \label{eqn:truncated-paths-superset}
    P(L_z, n) \subseteq P(L_z, n, d).
\end{align}
More consequentially, observe that we can easily compute $\Pr_R[P(L_z, n, d)]$ exactly. For this, we first check if either $d(n)=2$ or $L_z \not\id \mu(n_{d(n)-1})$. If this is not the case, then $\Pr_R[P(L_z, n, d)]=0$. Otherwise, we consider all the descendants of node $n$ at level $d$---there are precisely $2^{d-d(n)}$ many of these. From each descendant $n_d$, we trace up a path to the root---say $(n_d,n_{d-1},\dots,n_{d(n)},\dots,n_2, n_1)$, and check whether $L_z \id \mu(n_i)$ for every $i \ge d(n)$. We count the number of descendants $N_{z, n, d}$ that pass this check. Then, we can see that $P(L_z, n, d)$ is precisely the event that a random path passes through one of these descendants, giving us that $\Pr_R[P(L_z, n, d)]=\frac{N_{z, n,d}}{2^{d-1}}$. This, together with the observation that $P(L_z, n, d)$ and $P(L_z, n', d)$ are disjoint for any $n \neq n'$, imply that 
\begin{align}
    \label{eqn:fraction-of-leaves}
    \sum_{\substack{n \in \mcT_{\mcA, \sigma} \\  d(n) \le d}}\Pr_R[P(L_z, n, d)] = \frac{\left|\{n \in \mcT_{\mcA, \sigma}: d(n)=d, L_z \id \mu(n)\}\right|}{2^{d-1}},
\end{align}
where the RHS above is precisely the fraction of nodes at level $d$ in $\mcT_{\mcA, \sigma}$ that identify $L_z$.

We are now ready to state the main theorem of this subsection which equates deterministic and probabilistic list identification.

\begin{theorem}[Probabilistic List Identification $\equiv$ Deterministic List Identification]
    \label{thm:probabilistic-deterministic-equivalence}
    Let $\mcC=\{L_1,L_2,\dots\}$ be a countable collection of languages.
    \begin{enumerate}
        \item[(1)] If $\mcC$ is deterministically $k$-list identifiable in the limit, then $\mcC$ is probabilistically $k$-list identifiable in the limit with probability $p$, for any $p \in [0,1]$.
        \item[(2)] If $\mcC$ is probabilistically $k$-list identifiable in the limit with probability $p > \frac{k}{k+1}$, then $\mcC$ is deterministically $k$-list identifiable in the limit.
    \end{enumerate}
\end{theorem}
\begin{proof}
    The proof of (1) is immediate: if there exists a deterministic $k$-list identifier $\mcA'$ that identifies $\mcC$ in the limit, then the probabilistic $k$-list identifier $\mcA$, which on input $\sigma=(x_1,x_2,\dots)$, disregards any random bits, and simply outputs $\mcA(x_1,r_1,\dots,x_t,r_t)=\mcA'(x_1,\dots,x_t)$, probabilistically $k$-list identifies $\mcC$ in the limit with probability 1.

    Given a probabilistic $k$-list identifier $\mcA$ for $\mcC$ that succeeds with probability larger than $\frac{k}{k+1}$, we will now construct a deterministic $k$-list identifier $\mcA'$ for $\mcC$. We will specify $\mcA'$ by specifying its outputs on any given input enumeration $\sigma := (x_1,x_2,\dots)$. For any index $l \in \N$, let us define $\first_\mcC(l) := \min\{i \in \N: L_i=L_l\}$. That is, $\first_\mcC(l)$ is the index of the first occurrence of language $L_l$ in the collection $\mcC$.

    \renewcommand{\figurename}{Procedure}
    \Crefname{figure}{Procedure}{Procedures} 
    \begin{figure}[H]
        \begin{framed}
            \centering \textbf{Deterministic List Identifier $\mcA'$ from Probabilistic List Identifier $\mcA$} \\
            \begin{flushleft}  
                At time step $t$ on input $\sigma := (x_1,x_2,\dots)$:
            \end{flushleft}
            \begin{enumerate}
                \item[(a)] Let $\mu_1,\dots,\mu_{2^{t}}$ denote the $k$-lists output at all the nodes at level $t+1$ in $\mcT_{\mcA, \sigma}$.
                \item[(b)] Let $S$ be the \textit{multiset} of $\first_\mcC(l)$ values for every index $l$ in these lists. That is,
                \begin{align*}
                    S := \mathrm{multiset}\{\first_\mcC(l): l \in \mu_i, 1 \le i \le 2^{t}\}.
                \end{align*}
                \item[(c)] Return $\mcA'(x_1,\dots,x_t) := \topk(S)$, where $\topk(S)$ returns a list of the $k$ most frequently occurring indices in $S$.
            \end{enumerate}
        \end{framed}
        \caption{Obtaining a deterministic list identifier from a probabilistic list identifier.}
        \label{proc:probabilistic-to-deterministic}
    \end{figure}
    Let $\mcA$ be a probabilistic $k$-list identifier that identifies $\mcC$ in the limit with probability $p > \frac{k}{k+1}$; in particular, let $p=\frac{k}{k+1}+\eps$ for some $\eps > 0$. Fix any language $L_z \in \mcC$, and any enumeration $\sigma$ of it. From \eqref{eqn:list-identification-probability-summation}, we have that
    \begin{align*}
        \sum_{n \in \mcT_{\mcA, \sigma}}\Pr_R[P(L_z, n)] \ge \frac{k}{k+1}+\eps.
    \end{align*}
    Because all the summands in the summation above are non-negative, there exists a finite $d^\star \in \N$ such that for all $d \ge d^\star$,
    \begin{align*}
        \sum_{\substack{n \in \mcT_{\mcA, \sigma} \\ d(n) \le d}} \Pr_R[P(L_z, n)] \ge \frac{k}{k+1}+\frac{\eps}{2}.
    \end{align*}
    From \eqref{eqn:truncated-paths-superset}, this implies that for all $d \ge d^\star$,
    \begin{align}
        \sum_{\substack{n \in \mcT_{\mcA, \sigma} \\ d(n) \le d}} \Pr_R[P(L_z, n, d)] \ge \frac{k}{k+1}+\frac{\eps}{2}.
    \end{align}
    From \eqref{eqn:fraction-of-leaves}, this implies that for all $d \ge d^\star$, the fraction of nodes in $\mcT_{\mcA, \sigma}$ at level $d$ that identify $L_z$ is at least $\frac{k}{k+1}+\frac{\eps}{2}$.
    
    So, consider any time step $t \ge d^\star-1$. By the reasoning above, we have that of the $2^{t}$ lists $\mu_1,\dots,\mu_{2^{t}}$ considered by the deterministic list identifier $\mcA'$ in Step (a) of \Cref{proc:probabilistic-to-deterministic}, at least a $\frac{k}{k+1}+\frac{\eps}{2}$ fraction identify $L_z$. This means that the multiset $S$ constructed in Step (b) contains at least $2^{t} \cdot \left(\frac{k}{k+1}+\frac{\eps}{2}\right)$ copies of $\first_\mcC(z)$. We then claim that $\topk(S)$ necessarily contains $\first_\mcC(z)$. If not, there are $k$ distinct indices other than $\first_\mcC(z)$, each of which occur at least $2^{t} \cdot \left(\frac{k}{k+1}+\frac{\eps}{2}\right)$ times in $S$. Since $S$ also contains $2^{t} \cdot \left(\frac{k}{k+1}+\frac{\eps}{2}\right)$ copies of $\first_\mcC(z)$, this would mean that there are a total of at least $2^t \cdot \left(k + \frac{\eps(k+1)}{2}\right)$ numbers in $S$, which is not possible since $S$ is a concatenation of at most $2^t$ lists of size $\le k$. Thus, we have argued that $\topk(S) = \mcA'(x_1,\dots,x_t)$ contains $\first_\mcC(z)$ for every $t \ge d^\star -1$, implying that $\mcA'$ (deterministically) list-identifies $L_z$ in the limit on the input enumeration $\sigma$. Since the language $L_z$ and its enumeration $\sigma$ were arbitrarily chosen, we conclude that $\mcA'$ list-identifies $\mcC$ in the limit.
\end{proof}

\subsubsection{Probabilistic List Identification $\equiv$ List Identification on Infinite Draws}
\label{sec:probabilistic-statistical-equivalent}

In both the definitions  \Cref{def:list-identification-in-the-limit} and \Cref{def:probabilistic-list-identification-in-the-limit}, we considered the enumeration $\sigma=(x_i)_{i \in \N}$ of the target language $L_z$ to be any fixed worst-case enumeration. In this subsection, we consider the setting where the enumeration $(x_i)_{i \in \N}$ is itself an infinite i.i.d.\ sequence drawn from a distribution $\mcD$ that is valid for the language $L_z$ (and the identifier is deterministic). Since $\mcD$ is valid for $L_z$, meaning that its support is precisely $L_z$, every member of $L_z$ will eventually show up in an infinite draw from $\mcD$, and there are no spurious strings in it. More formally, Proposition 5.2 in \cite{kalavasis2025limits} shows that such an infinite draw is an enumeration of $L_z$ with probability 1.

\begin{definition}[List Identification in the Limit on Infinite Draws]
    \label{def:list-identification-in-the-limit-on-infinite-draws}
    A deterministic $k$-list identifier $\mcA'$ identifies a countable collection $\mcC=\{L_1,L_2,\dots\}$ in the limit with probability $p$ on infinite draws if for every language $L_z\in \mcC$, and for every distribution $\mcD$ that is valid for $L_z$, with probability at least $p$ over $(x_i)_{i \in \N} \sim \mcD^\infty$, there exists a finite time $t^\star$ such that for every $t \ge t^\star$, the list $\mu_t=\mcA'(x_1,\dots,x_t)$ output by $\mcA'$ satisfies that $L_z \id \mu_t$ 
\end{definition}

The main theorem of this subsection equates list identification in the limit on infinite draws and probabilistic list identification in the limit. The proof of this theorem follows the proof structure of Theorem 9 in \cite{angluin1988identifying}.

\begin{theorem}[Probabilistic List Identification $\equiv$ List Identification on Infinite Draws]
    \label{thm:probabilistic-statistical-equivalence}
    Let $\mcC=\{L_1,L_2,\dots\}$ be a countable collection of languages, and fix $p \in [0,1]$. Then, $\mcC$ is $k$-list identifiable in the limit with probability $p$ on infinite draws if and only if $\mcC$ is probabilistically $k$-list identifiable in the limit with probability $p$.
\end{theorem}
\begin{proof}
    First, let $\mcA'$ be a deterministic $k$-list identifier that identifies $\mcC$ in the limit with probability $p$ on infinite draws. We will construct a probabilistic $k$-list identifier $\mcA$ for $\mcC$ that succeeds with probability $p$. Fix any language $L_z$ and any enumeration $\sigma := (x_i)_{i \in \N}$ of it. Define a distribution $\mcD_\sigma$ supported over $L_z$ as follows:
    \begin{align*}
        \mcD_\sigma(x) = \begin{cases}
            \sum_{i \in \N: x_i=x}\frac{1}{2^i} & \text{if } x \in L_z, \\
            0 & \text{otherwise}.
        \end{cases}
    \end{align*}
    Since $\sigma$ is a valid enumeration of $L_z$, we can verify that $\mcD_\sigma$ is a valid distribution for $L_z$. Then, consider the probabilistic $k$-list identifier $\mcA$ which, given $\sigma$ and an infinite random bit sequence $R := (r_i)_{i \in \N}$ as input, operates as follows. Using $R$, $\mcA$ constructs an infinite draw $\sigma' := (x'_i)_{i \in \N} \sim \mcD_\sigma$ for $\mcD_\sigma$ defined above.\footnote{For example, one way to do this is as follows: to generate $x' \sim \mcD_\sigma$, $\mcA$ scans (the rest of) $R$ ahead and finds the position $j$ of the first 1 it encounters (here, we imagine the first bit in the rest of $R$ to be at position 1). $\mcA$ then sets $x'$ to be $x_j$ in $\sigma$. We can verify that over the randomness of $R$, $x'$ generated thus is distributed according to $\mcD_\sigma$.} It then feeds $\sigma'$ to $\mcA'$, and at each time step, outputs the list output by $\mcA'$. Since $\mcD_\sigma$ is a valid distribution for $L_z$, we have that with probability at least $p$ over $\sigma'$, $\mcA'$ will converge to outputting a list that identifies $L_z$. Since the draw of $\sigma'$ is induced by the draw of $R$, and $\mcA$ returns the output of $\mcA'$, we conclude that with probability at least $p$ over $R$, $\mcA$ converges to outputting a list that identifies $L_z$ on the enumeration $\sigma$.

    Now, let $\mcA$ be a probabilistic $k$-list identifier that identifies $\mcC$ in the limit with probability $p$. We will construct a deterministic $k$-list identifier $\mcA'$ that identifies $\mcC$ in the limit with probability $p$ on infinite draws. Fix any language $L_z \in \mcC$, and any distribution $\mcD$ that is valid for $L_z$. Consider $\mcA'$, which on input $(x_i)_{i \in \N} \sim \mcD^\infty$, operates as follows. Suppose $x_1=a$: $\mcA'$ scans increasing prefixes $(x_1,\dots,x_t)$ of the input up until the first time it sees some $x_t \neq a$. Up until this time, $\mcA'$ keeps outputting the singleton list $\{L_j\}$, where $j$ is the index of any language in $\mcC$ for which $L_j = \{a\}$; if no such language exists $\mcA'$ keeps outputting $\{L_1\}$ arbitrarily until it arrives at $x_t$. Observe that if $L_z$ were a singleton set, then the only valid distribution for $L_z$ is a point mass on the single element in $L_z$, and $\mcA'$ identifies $L_z$ with probability 1 from the first time step.
    
    Otherwise, $\mcA'$ reaches an $x_t =b\neq a$. Note that this can only happen if $L_z$ has at least two elements. Furthermore, in this case, since $\mcD$ is supported on all of $L_z$, such an $x_t$ will be encountered with probability 1. $\mcA'$ then feeds $(x_{2i-1})_{i \in \N}$ to $\mcA$ as input. Since $(x_{2i-1})_{i \in \N}$ is also an infinite draw from $\mcD$, it is a valid enumeration of $L_z$ with probability 1. $\mcA'$ constructs a random bit sequence $(r_i)_{i \in \N}$ for $\mcA$, using the randomness in the input, in the following manner: denote $t=j_1$, where $t$ was the first time where $x_t = b$.
    In order to determine $r_1$, $\mcA'$ scans pairs $(x_{2(j_1+1)}, x_{2(j_1+2)}), (x_{2(j_1+3)}, x_{2(j_1+4)}),\dots$ up until the time it sees a pair which is either $(a,b)$ or $(b,a)$---say this happens at $(x_{2(j_1+m)}, x_{2(j_1+m+1)})$. This process also terminates with probability 1, and once it terminates, $r_1$ is set to $1$ if the process terminates with the pair $(a,b)$ and $0$ if it terminates with $(b,a)$. Thereafter, $r_2$ is determined similarly by setting $j_2=j_1+m+1$, and starting to scan pairs $(x_{2(j_2+1)}, x_{2(j_2+2)}), (x_{2(j_2+3)}, x_{2(j_2+4)}),\dots$ and so on.

    Since $(x_i)_{i \in \N}$ is an infinite i.i.d. draw from $\mcD$, conditioned on any $(x_{2i-1})_{i \in \N}$, the sequence $(x_{2i})_{i \in \N}$, is also an infinite i.i.d. draw from $\mcD$. So, consider the distribution of $(r_i)_{i \in \N}$ induced by the randomness in $(x_{2i})_{i \in \N} \sim \mcD^\infty$. Since each $r_i$ is determined using independent strings from $\mcD$, the bits are independent. Furthermore, conditioned on the first two distinct strings in $(x_{2i})_{i \in \N}$ being any $(a,b)$, the next pair in the sequence that is one of $(a,b)$ or $(b,a)$ is equally likely to be either. So, every $r_i$ is either 0 or 1 with probability $1/2$ each. Thus, we have argued that the distribution of $(r_i)_{i \in \N}$ induced by $(x_{2i})_{i \in \N} \sim \mcD^\infty$ is precisely that of an infinite sequence of independent random bits. 

    So, $\mcA'$ feeds $(x_{2i-1})_{i \in \N}$ to the probabilistic $k$-list identifier $\mcA$ as input, along with the bit sequence $(r_i)_{i \in \N}$ that it constructs from $(x_{2i})_{i \in \N}$. At each time step, it outputs precisely the list that is output by $\mcA$. Observe that the outputs of $\mcA'$ are deterministic. It remains to argue that $\mcA'$ thus constructed identifies $L_z$ with probability at least $p$ over the draw of $(x_i)_{i \in \N}$. To see this, note that
    \begin{align*}
        &\Pr_{(x_i)_{i \in \N} \sim \mcD^\infty}\left[\text{$\mcA'$ list-identifies $L_z$ in the limit on input $(x_i)_{i \in \N}$}\right] \\
        =& \E_{(x_{2i-1})_{i \in \N} \sim \mcD^\infty}\left[ \Pr_{(x_{2i})_{i \in \N} \sim \mcD^\infty}\left[\text{$\mcA'$ list-identifies $L_z$ in the limit on input $(x_i)_{i \in \N}$} ~\Big|~ (x_{2i-1})_{i \in \N} \right]\right] \\
        =& \E_{(x_{2i-1})_{i \in \N} \sim \mcD^\infty}\left[ \Pr_{(r_i)_{i \in \N}}\left[\text{$\mcA$ list-identifies $L_z$ in the limit on input $(x_{2i-1})_{i \in \N}$} ~\Big|~ (x_{2i-1})_{i \in \N} \right]\right] \\
        \ge &\E_{(x_{2i-1})_{i \in \N} \sim \mcD^\infty}\left [ p\right] = p.
    \end{align*}
    The inequality above follows from our above reasoning, which establishes that the induced distribution of $(r_i)_{i \in \N}$ conditioned on any $(x_{2i-1})_{i \in \N}$ is that of an infinite sequence of independent random bits, together with the guarantee that $\mcA$, by virtue of being a probabilistic $k$-list identifier, identifies $L_z$ in the limit with probability at least $p$ over the randomness in $(r_i)_{i \in \N}$ on the (fixed) enumeration $(x_{2i-1})_{i \in \N}$ of $L_z$.
\end{proof}

\subsubsection{Condition Not Satisfied $\implies$ No Rate}
\label{sec:condition-not-satisfied-implies-no-rate}

As a corollary of \Cref{thm:k-angluin-condition-lb,thm:probabilistic-deterministic-equivalence,thm:probabilistic-statistical-equivalence}, we have the following:

\begin{corollary}
    \label{corollary:probabilistic-statistical}
    Let $\mcC$ be a countable collection of languages that does not satisfy the $k$-Angluin condition \eqref{eqn:condition}. Then, for every (deterministic) $k$-list identifier $\mcA$, there exists a language $L_z \in \mcC$ and a distribution $\mcD$ that is valid for $L_z$, such that $\mcA$ fails to $k$-list identify $L_z$ in the limit on $(x_i)_{i \in \N} \sim \mcD^\infty$ with probability at least $\frac{1}{k+1}$. In other words, 
    \begin{align*}
        \Pr_{(x_i)_{i \in \N} \sim \mcD^\infty}\left[\exists i_1 < i_2 < i_3 < \dots : \forall j \in \N, \mcA(x_1,x_2,\dots,x_{i_j}) \not\di L_z \right] \ge \frac{1}{k+1}.
    \end{align*}
\end{corollary}

Similar to \cite{kalavasis2025limits}, we now have all the ingredients to establish \Cref{thm:no-rate}, which shows that any collection that does not satisfy the $k$-Angluin condition \eqref{eqn:condition} cannot be $k$-list identified at any rate. More precisely, we will show that:

\begin{theorem}
    \label{thm:no-condition-no-rate-quantitative}
    Let $\mcC$ be a countable collection of languages that does not satisfy the $k$-Angluin condition \eqref{eqn:condition}. Then, for every (deterministic) $k$-list identifier $\mcA$, there exists a language $L_z \in \mcC$ and a distribution $\mcD$ that is valid for $L_z$, such that
    \begin{align}
        \limsup_{t \to \infty} \Pr_{x_1,\dots,x_t \sim \mcD^t}\left[\mcA(x_1,\dots,x_t) \not\di L_z \right] \ge \frac{1}{k+1}.
    \end{align}
\end{theorem}
\begin{proof}
    Assume for the sake of contradiction that there exists a $k$-list identifier $\mcA$ which satisfies that for every language $L_z \in \mcC$, and for every distribution $\mcD$ that is valid for $L_z$,
    \begin{align}
        \label{eqn:to-contradict}
        \limsup_{t \to \infty} \Pr_{x_1,\dots,x_t \sim \mcD^t}\left[\mcA(x_1,\dots,x_t) \not\di L_z \right] \le \frac{1}{k+1}-\eps
    \end{align}
    for some $\eps > 0$. This means that for every $L_z \in \mcC$, and every distribution $\mcD$ that is valid for $L_z$, there exists a finite $t_0=t_0(L_z, \mcD)$, such that for every $t \ge t_0$, 
    \begin{align}
        \label{eqn:prob-not-list-identify-small}
        \Pr_{x_1,\dots,x_t \sim \mcD^t}\left[\mcA(x_1,\dots,x_t) \not\di L_z \right] \le \frac{1}{k+1}-\frac{\eps}{2}.
    \end{align}
    Otherwise, if there were infinitely many $t$ for which $\Pr_{x_1,\dots,x_t \sim \mcD^t}\left[\mcA(x_1,\dots,x_t) \not\di L_z \right] > \frac{1}{k+1}-\frac{\eps}{2}$, then we would have $\limsup_{t \to \infty} \Pr_{x_1,\dots,x_t \sim \mcD^t}\left[\mcA(x_1,\dots,x_t) \not\di L_z \right] \ge \frac{1}{k+1}-\frac{\eps}{2}$, which would contradict \eqref{eqn:to-contradict}.

    So, with a view to contradict \Cref{corollary:probabilistic-statistical}, let us fix an arbitrary $L_z \in \mcC$, and an arbitrary $\mcD$ that is valid for $L_z$. We will use \eqref{eqn:prob-not-list-identify-small} to construct the following $k$-list identifier $\mcA'$: given $x_1,\dots,x_t \sim \mcD^t$, $\mcA'$ splits it into $M := \frac{t}{\log t}$ batches $B_1,\dots,B_M$, each of size $\log t$. It then runs $\mcA$ on each batch to obtain the lists $\mcA(B_1),\dots,\mcA(B_M)$. Then, it constructs $S$ to be the multiset of $\first_\mcC(l)$ values for every index $l$ in these lists, i.e.,
    \begin{align*}
        S := \mathrm{multiset}\left\{\first_\mcC(l) : l \in \mcA(B_j), 1 \le j \le M\right\}.
    \end{align*}
    Finally, $\mcA'$ outputs $\topk(S)$. These latter steps are similar to Steps (b) and (c) in \Cref{proc:probabilistic-to-deterministic}. In essence, we are boosting the guarantee in \eqref{eqn:prob-not-list-identify-small} by combining the predictions of $\mcA$ on independent batches of increasing size.

    Observe then that when $t \ge e^{t_0}$ for $t_0=t_0(L_z, \mcD)$ satisfying \eqref{eqn:prob-not-list-identify-small}, we will have that every batch $B_j$ is of size at least $t_0$, which implies that
    \begin{align*}
        \Pr_{B_j}\left[\mcA(B_j) \di L_z\right] \ge \frac{k}{k+1}+\frac{\eps}{2}.
    \end{align*}
   Fix any $t \ge e^{t_0}$, and let $Y_j = \Ind[\mcA(B_j) \di L_z]$. Since the batches are independent, we have that the $Y_j$s are independent, and also, $\E\left[\sum_{j=1}^M Y_j\right] \ge M\left(\frac{k}{k+1}+\frac{\eps}{2}\right)$. Then, we have that
   \begin{align*}
        \Pr_{x_1,\dots,x_t \sim \mcD^t}\left[\sum_{j=1}^M Y_j \le M\left(\frac{k}{k+1}+\frac{\eps}{4}\right)\right] &\le \Pr_{x_1,\dots,x_t \sim \mcD^t}\left[\sum_{j=1}^M Y_j \le \E\left[\sum_{j=1}^M Y_j\right] - \frac{M\eps}{4}\right] \\
        &\le e^{-M\eps^2/8}. \tag{Hoeffding's Inequality}
   \end{align*}
   Therefore, we have that with probability at least $1-e^{-M\eps^2/8}$, it holds that $\frac{1}{M}\sum_{j=1}^M Y_j > \frac{k}{k+1}+\frac{\eps}{4}$, which means that of the lists $\mcA(B_1),\dots,\mcA(B_M)$, strictly more than a $\frac{k}{k+1}+\frac{\eps}{4}$ fraction identify $L_z$. By a similar argument as that at the end of the proof of \Cref{thm:probabilistic-deterministic-equivalence}, this implies that $\topk(S)$ contains $\first_\mcC(z)$. Summarily, we have argued that for every $t \ge e^{t_0}$,
   \begin{align*}
        \Pr_{x_1,\dots,x_t \sim \mcD^t}\left[\mcA'(x_1,\dots,x_t) \not\di L_z\right] \le \Pr_{x_1,\dots,x_t \sim \mcD^t}\left[\first_\mcC(z) \not\in \mcA'(x_1,\dots,x_t)\right] \le e^{-M\eps^2/8} \le e^{-\frac{t\eps^2}{8\log t}}.
   \end{align*}
   Thus, we have that
   \begin{align*}
        \sum_{t \in \N} \Pr_{(x_i)_{i \in \N} \sim \mcD^\infty}\left[\mcA'(x_1,\dots,x_t) \not\di L_z\right] &=
        \sum_{t \in \N} \Pr_{x_1,\dots,x_t \sim \mcD^t}\left[\mcA'(x_1,\dots,x_t) \not\di L_z\right] \\
        &\le e^{t_0} + \sum_{t \ge e^{t_0}} \Pr_{x_1,\dots,x_t \sim \mcD^t}\left[\mcA'(x_1,\dots,x_t) \not\di L_z\right] \\
        &\le e^{t_0} + \sum_{t \ge e^{t_0}}e^{-\frac{t\eps^2}{8\log t}} < \infty.
   \end{align*}
   Then, by the Borel-Cantelli lemma, we have that
   \begin{align*}
        \Pr_{(x_i)_{i \in \N} \sim \mcD^\infty}\left[\exists i_1 < i_2 < i_3 < \dots : \forall j \in \N, \mcA'(x_1,x_2,\dots,x_{i_j}) \not\di L_z \right] = 0.
   \end{align*}
   Since the language $L_z$ and distribution $\mcD$ were arbitrarily chosen, $\mcA'$ satisfies this conclusion for every $L_z \in \mcC$ and every distribution $\mcD$ valid for $L_z$. But since $\mcC$ does not satisfy the $k$-Angluin condition \eqref{eqn:condition}, this contradicts \Cref{corollary:probabilistic-statistical}, concluding the proof.    
\end{proof}

\subsection{Condition Satisfied $\implies$ Exponential Rate}
\label{sec:exponential-rate}

We will now argue that any collection that satisfies the $k$-Angluin condition \eqref{eqn:condition} can be $k$-list identified at an exponential rate. While this can be shown by arguing that the list identification algorithm from \Cref{sec:ub} achieves exponential rates, we will provide a more direct argument that uses the structural property of a $k$-list identifiable collection, which allows it to be decomposed into $k$ identifiable collections.

\begin{theorem}[Expnential Rate]
    \label{thm:exponential-rate}
    Let $\mcC$ be a countable collection of languages that satisfies the $k$-Angluin condition \eqref{eqn:condition}. Then, $\mcC$ can be $k$-list identified at rate $R(t)=e^{-t}$.
\end{theorem}
\begin{proof}
    Since $\mcC$ satisfies the $k$-Angluin condition, \Cref{thm:k-angluin-condition-ub} ensures that it is $k$-list identifiable in the limit. Then, from \Cref{thm:stratification}, we further know that $\mcC$ may be expressed as $\mcC = \cup_{i=1}^k\mcC_i$, where each $\mcC_i$ is identifiable in the limit, meaning that it satisfies Angluin's condition for identification (the base predicate given in \eqref{eqn:base-predicate}). From Proposition 3.10 in \cite{kalavasis2025limits}, we then know that there exists an identifier $\mcA_i$ for every $\mcC_i$ that identifies languages in $\mcC_i$ at an exponential rate. This immediately implies that the $k$-list identifier $\mcA$, which concatenates the outputs of $\mcA_1,\dots,\mcA_k$ that are run on the collections $\mcC_1,\dots,\mcC_k$ respectively, achieves an exponential rate for $k$-list identification of $\mcC$. 
\end{proof}

\subsection{Exponential Rate Best Possible}
\label{sec:exponential-rate-best-possible}

We will now show that an exponential rate is effectively the best rate possible for list identification, upto a technical triviality condition. This technical condition captures the following triviality: if a collection $\mcC$ satisfies that for every $k+1$ distinct languages $L_{i_1},\dots,L_{i_{k+1}}$ in the collection, $\cap_{j=1}^{k+1} L_{i_j}=\emptyset$, then observe that any single string from the target language $L_z$ pins down a list of at most $k$ distinct languages that must necessarily identify $L_z$. That is, collections satisfying this property can be list-identified right from $t=1$ (i.e., at a zero failure probability rate). So, adopting terminology from \cite{kalavasis2025limits}, we will say that a collection is \textit{non-trivial} for $k$-list identification if there exist $k+1$ distinct languages $L_{i_1},\dots,L_{i_{k+1}}$ in $\mcC$ which satisfy $\cap_{j=1}^{k+1} L_{i_j} \neq \emptyset$. The next theorem shows that a collection that is non-trivial for $k$-list identification cannot be $k$-list identified at a rate that is faster than exponential.

\begin{theorem}[Exponential Rate Best Possible]
    \label{thm:exponential-rate-best-possible}
    Let $\mcC$ be a countable collection of languages that is non-trivial for $k$-list identification. Then, for every $k$-list identifier $\mcA$, there exists a language $L_z \in \mcC$ and a distribution $\mcD$ valid for $L_z$ such that $\Pr_{x_1,\dots,x_t \sim \mcD^t}[\mcA(x_1,\dots,x_t) \not\di L_z] \ge e^{-2t}$ for infinitely many $t \in \N$.
\end{theorem}
\begin{proof}
    Given that $\mcC$ is non-trivial for $k$-list identification, there exist $k+1$ distinct languages $L_{i_1},\dots,L_{i_{k+1}}$ in $\mcC$ that all share some string $x$. For each $j$, let $\mcD_{i_j}$ be the distribution that assigns mass at least $1/2$ to $x$, and distributes the rest of the mass among the rest of the strings in $L_{i_j}$ arbitrarily, in a way that every string gets assigned a positive mass. Then, for every $\mcD_{i_j}$, for any $t \in \N$, there is at least a $1/2^t$ chance that the first $t$ strings $x_1,\dots,x_t \sim \mcD_{i_j}^t$ are all equal to $x$. Denoting this event by $\mcE_t$, we have that for every $j$, $\Pr_{x_1,\dots,x_t \sim \mcD_{i_j}^t}[\mcE_t] \ge \frac{1}{2^t}$. Observe now that conditioned on $\mcE_t$, because $\mcA(x_1,\dots,x_t)$ has size at most $k$,
    \begin{align*}
        \sum_{j=1}^{k+1}\Ind[\mcA(x_1,\dots,x_t) \di L_{i_j}] \le k,\footnotemark
    \end{align*}
    \footnotetext{If $\mcA$ is allowed to be randomized, we can instantiate this inequality as $\sum_{j=1}^{k+1}\Pr[\mcA(x_1,\dots,x_t) \di L_{i_j} | \mcE_t]=\E\left[\sum_{j=1}^{k+1}\Ind[\mcA(x_1,\dots,x_t) \di L_{i_j}] ~\big|~ \mcE_t\right] \le k$, where the probability is only over the randomness in $\mcA$.}which means that for some $j$, $\Ind[\mcA(x_1,\dots,x_t) \not\di L_{i_j}] = 1$. By the pigeonhole principle, it holds that for some $j$, $\Ind[\mcA(x_1,\dots,x_t) \not\di L_{i_j}] = 1$ conditioned on $\mcE_t$, for infinitely many $t \in \N$. We can therefore conclude that for such a $j$, for infinitely many $t$,
    \begin{align*}
        \Pr_{x_1,\dots,x_t \sim \mcD_{i_j}^t}[\mcA(x_1,\dots,x_t) \not\di L_{i_j})] &\ge \Pr_{x_1,\dots,x_t \sim \mcD_{i_j}^t}[\mcE_t] \cdot \Pr_{x_1,\dots,x_t \sim \mcD_{i_j}^t}[\mcA(x_1,\dots,x_t) \not\di L_{i_j}) ~|~ \mcE_t] \\
        &\ge 2^{-t} \ge e^{-2t}.
    \end{align*}
\end{proof}

\section*{Acknowledgements}
This work was supported by Moses Charikar's and Gregory Valiant's Simons Investigator Awards, and a Google PhD Fellowship.

\bibliographystyle{alpha} 
\bibliography{references}

\end{document}